\documentclass[11pt]{article}

\usepackage{acl}

\usepackage{times}
\usepackage{latexsym}

\usepackage[T1]{fontenc}
\usepackage[utf8]{inputenc}

\usepackage{microtype}

\usepackage{amsmath,amssymb,amsfonts}
\usepackage{amsthm}
\usepackage{algorithm}
\usepackage{algorithmic}
\usepackage{graphicx}
\usepackage{xcolor}
\usepackage{booktabs}

\newtheorem{theorem}{Theorem}

\newtheorem{proposition}[theorem]{Proposition}

\newtheorem{definition}{Definition}

\begin{document}

\title{GAC: Noise-Aware Adaptive Mixing for Hybrid SFT-RL Post-Training}

\author{Yuelin Hu$^1$ \quad Wei Liu$^2$ \quad Zhenbo Yu$^{1,3}$\thanks{Co-corresponding authors.} \quad Zhengxue Cheng$^1$ \quad Li Song$^1$\footnotemark[\value{footnote}]\\
$^1$Shanghai Jiao Tong University \quad $^2$Shanghai Maritime University \quad $^3$Novacore\\
\texttt{\{huyuelin51717221,yuzhenbo,zxcheng,songli\}@sjtu.edu.cn}}

\hypersetup{
  pdftitle={GAC: Noise-Aware Adaptive Mixing for Hybrid SFT-RL Post-Training},
  pdfauthor={Yuelin Hu, Wei Liu, Zhenbo Yu, Zhengxue Cheng, Li Song},
  pdfsubject={Hybrid SFT-RL post-training for large language models},
  pdfkeywords={post-training, supervised fine-tuning, reinforcement learning, gradient noise, adaptive mixing}
}

\maketitle

\begin{abstract}
Hybrid post-training combining supervised fine-tuning (SFT) and reinforcement learning (RL) is the standard paradigm for aligning large language models, yet fixed mixing schedules cannot adapt when the relative noise of the two signals evolves. We derive a noise-aware mixing weight $\mu^*$ by minimizing an MSE upper bound on the mixed stochastic gradient, yielding a closed-form that balances gradient noise variance and SFT--RL disagreement. Building on the token-wise reweighting of CHORD \cite{zhang2025chord}, the practical Guided Adaptive Controller (GAC) adds EMA smoothing, a schedule prior, and capped updates around this estimator, with all statistics estimated online from existing training tensors. The noise-aware controller alone outperforms the best rule-based controller by $+$3.0pp on AMC; the full system reaches $+$3.8pp over HPT across math, code, science, and logic benchmarks, while reducing KL-drift area by 28\% and large $|\Delta\mu|$ events by $>$70\%, at $<$1\% overhead. Gains grow with model size from 1.5B to 14B (Table~\ref{tab:scale}). Resources: a Qwen3.5-4B compact-backbone release is available on \href{https://huggingface.co/YueLinHu/GAC-Qwen3.5-4B}{Hugging Face}; the reference implementation is at \href{https://github.com/deepnovacore/GAC}{GitHub}; and the \href{https://deepnovacore.github.io/GAC/}{GAC project page} provides project details.
\end{abstract}

\begin{figure*}[!t]
  \centering
  \includegraphics[width=\textwidth]{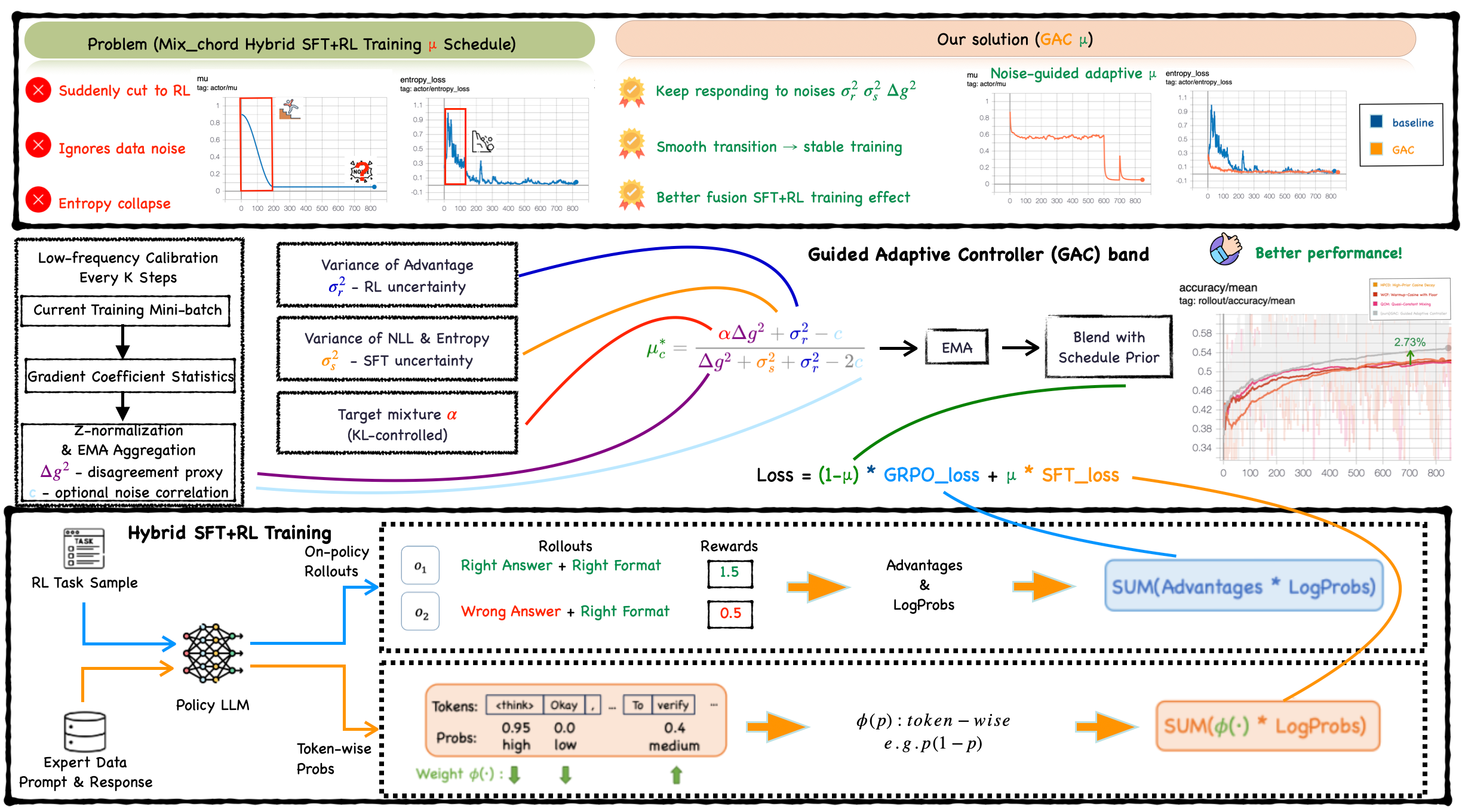}
  \caption{Overview of GAC and the training pipeline. Top: limitations of fixed schedules versus the noise-aware controller. Middle: the Guided Adaptive Controller integrates $\sigma_r^2$ (RL uncertainty), $\sigma_s^2$ (SFT uncertainty), and $\Delta \tilde g^2$ (disagreement proxy) to compute the mixing weight in \eqref{eq:mu_star}, followed by EMA smoothing, prior blending, and capped updates. Bottom: the pipeline sources signals from GRPO and SFT with token-wise weighting $\varphi(p){=}p(1{-}p)$ \cite{zhang2025chord}, composing $L{=}(1{-}\mu)L_{\mathrm{GRPO}}{+}\mu L_{\mathrm{SFT}{-}\varphi}$. All statistics are estimated online from existing training tensors.}
  \label{fig:overview}
\end{figure*}

\section{Introduction}

Large language models (LLMs) are typically post-trained by combining supervised fine-tuning (SFT) with reinforcement learning (RL). SFT stabilizes format from expert demonstrations while RL improves reward-seeking behavior from on-policy rollouts. However, SFT and RL objectives cannot be fully decoupled without mutual degradation \cite{niu2026nondecoupling}, and fixed mixtures cannot adapt to evolving policy drift and reward noise, inducing entropy collapse or late over-imitation. Recent hybrid methods address this via interleaved SFT--RL \cite{su2025trapo}, anchored regularization \cite{zhu2025asft}, or conflict-aware coupling \cite{zeng2025gta}. We take a complementary approach: a noise-aware global controller that adapts the mixing weight based on online gradient uncertainty estimates, combined with the token-wise stabilization $\varphi(\cdot)$ from CHORD \cite{zhang2025chord}. Figure~\ref{fig:overview} provides an overview.

\paragraph{Contributions}
The new contribution is a global noise-aware mixing controller and its proxy-based online estimation; the token-wise reweighting $\varphi(p)$ is adopted from CHORD \cite{zhang2025chord}. (C1) A closed-form $\mu^*$ from MSE minimization (Eq.~\ref{eq:mu_star}), instantiated via z-normalized proxies (Appendix~\ref{app:proxy_validation}), wrapped in a guided controller with motivating stability analysis (Proposition~\ref{prop:lyapunov}), reusing existing tensors with $<$1\% overhead. (C2) Token-wise SFT reweighting $\varphi(p){=}p(1{-}p)$ adopted from CHORD, contributing +0.4--1.4pp orthogonal to the controller (Section~\ref{sec:ablations}). (C3) Systematic evaluation across math, code, science, and logic at 1.5B, 7B, and 14B scales.

\section{Related Work}
\textbf{Post-training paradigms.} Sequential SFT-then-RL exhibits ``shift--readapt--overfit'' progression; scheduled mixtures remain heuristic \cite{instructgpt,christiano2017deep,rafailov2023dpo}. \citet{niu2026nondecoupling} prove that SFT and RL cannot be decoupled without mutual degradation, motivating integrated training.

\textbf{Dynamic weighting and stability.} Multi-task learning employs uncertainty weighting \cite{kendall}, gradient normalization \cite{gradnorm}, or conflict resolution (MGDA/PCGrad/CAGrad) \cite{mgda,pcgrad,cagrad}. Recent approaches include Nash-MTL \cite{nashmtl}, FAMO \cite{famo}, Aligned-MTL \cite{alignedmtl}, SDMGrad \cite{sdmgrad}, and MoCo \cite{moco}. None explicitly models both gradient noise variance and SFT--RL disagreement. Table~\ref{tab:method_compare} summarizes key differences.

\begin{table}[t]
\centering
\caption{Comparison of dynamic weighting methods for SFT--RL mixing. $\dagger$: derived but omitted due to high estimation variance. $\ddagger$: loss-ratio dynamics, not explicit EMA.}
\vspace{2pt}
\setlength{\tabcolsep}{2pt}
\footnotesize
\resizebox{0.96\linewidth}{!}{%
\begin{tabular}{lccccc}
\toprule
Method & Noise & Disagree. & Corr. & Temporal & Closed \\
& $\sigma^2$ & $\Delta g^2$ & $c$ & Smooth & Form \\
\midrule
Uncertainty \cite{kendall} & \checkmark & & & & \checkmark \\
GradNorm \cite{gradnorm} & & \checkmark & & & \\
MGDA/PCGrad & & \checkmark & & & \\
Nash-MTL \cite{nashmtl} & & \checkmark & & & \\
DWA \cite{dwa} & & & & $\ddagger$ & \\
FAMO \cite{famo} & & \checkmark & & \checkmark & \checkmark \\
Aligned-MTL \cite{alignedmtl} & & \checkmark & & & \\
SDMGrad \cite{sdmgrad} & & \checkmark & & & \\
MoCo \cite{moco} & & \checkmark & & \checkmark & \\
\midrule
CHORD \cite{zhang2025chord} & & & & \checkmark & \\
SRFT \cite{fu2025srft} & & & & & \\
LUFFY \cite{yan2025luffy} & & \checkmark & & & \\
HPT \cite{lv2025hpt} & & \checkmark & & \checkmark & \\
TRAPO \cite{su2025trapo} & & & & \checkmark & \\
ASFT \cite{zhu2025asft} & & & & & \\
\midrule
\textbf{GAC (ours)} & \checkmark & \checkmark & $\dagger$ & \checkmark & \checkmark \\
\bottomrule
\end{tabular}%
}
\label{tab:method_compare}
\end{table}

\textbf{Hybrid SFT--RL post-training.} CHORD \cite{zhang2025chord} proposes dual control with a global $\mu$ and token-wise $\varphi(p){=}p(1{-}p)$; GAC builds on CHORD, replacing its heuristic schedule with a noise-aware controller. SRFT \cite{fu2025srft} uses entropy-aware weighting; LUFFY \cite{yan2025luffy} augments RL with off-policy traces; HPT \cite{lv2025hpt} derives accuracy-gated signal selection. TRAPO \cite{su2025trapo} interleaves SFT and RL within each instance via trust-region SFT. ASFT \cite{zhu2025asft} anchors the policy to the base distribution via KL regularization. GTA \cite{zeng2025gta} combines supervised and RL signals with conflict mitigation. GAC differs by operating at the gradient noise level with a closed-form $\mu^*$ (Proposition~\ref{prop:lyapunov}).

\section{Method}

\noindent\textbf{Notation.} We write $\mu^*$ for the idealized optimal mixing weight, $\sigma_s^2,\sigma_r^2$ for SFT/RL gradient noise variance (estimated via proxies), $\Delta g^2$ for gradient disagreement, and $\alpha_{\mathrm{tgt}}$ for the theoretical target ratio in the MSE derivation. In practice we use a KL-controlled ratio $\alpha_{\mathrm{ctrl}}$ (Eq.~\ref{eq:alpha}). Throughout, ``SFT'' denotes $L_{\mathrm{SFT}{-}\varphi}$ with token-wise weighting $\varphi(p){=}p(1{-}p)$ \cite{zhang2025chord}. A full notation table is provided in Appendix~\ref{app:mse_derivation}.

\paragraph{Scope of the theoretical claims.}
Before deriving the closed-form controller, we make its regime of applicability explicit.
(i) \emph{Unbiased, variance-dominated setting.} The MSE objective in \eqref{eq:mse_def} treats $\hat g_s,\hat g_r$ as unbiased estimators of $g_s^*,g_r^*$ and models their stochasticity through the variances $\sigma_s^2,\sigma_r^2$; \eqref{eq:mu_star} is therefore the minimum-MSE mixture when estimator bias is negligible relative to variance. This is a defensible operating point for the verifiable-reward GRPO regime we study, where the reward signal is rule-based and biases enter primarily through PPO-style clipping and length normalization.
(ii) \emph{Bias-dominated regimes.} When estimator bias dominates the noise budget (e.g., aggressive advantage clipping, learned reward models with systematic drift), the same closed form no longer inherits a global optimality interpretation. In such regimes, $\mu^*$ should be read as a \emph{variance-aware stabilizing rule} that continues to weight the two objectives by their relative reliability, rather than as an exact MSE optimum. Appendix~\ref{app:biased_estimator} carries the derivation through the bias-augmented MSE upper bound (Eq.~\ref{eq:mu_star_bias}), which reduces to \eqref{eq:mu_star} when $b_s{=}b_r{=}0$.
(iii) \emph{Proxy semantics.} The controller does not consume gradient-level quantities directly. It consumes three \emph{coefficient-space} proxies: (a) $\sigma_r^2$ is the dispersion of \emph{post-normalization} sequence-level advantages (Eq.~\ref{eq:sigma_r}), not raw reward variance; (b) $\sigma_s^2$ is the trimmed variance of \emph{length-normalized} per-sequence NLL (Eq.~\ref{eq:sigma_s}); (c) $\Delta\tilde g^2$ is a coefficient-space mismatch on shared response tokens (Eq.~\ref{eq:delta_g}), not a parameter-space gradient inner product. Their empirical validity as gradient-statistic surrogates is quantified in Section~\ref{sec:ablations} (Tables~\ref{tab:proxy_corr}--\ref{tab:proxy_sanity}) rather than assumed a priori.

\subsection{Closed-Form \texorpdfstring{$\mu$}{mu} via MSE Minimization}

Let SFT and RL provide gradient estimators $\hat g_s=g_s^*+\varepsilon_s$ and $\hat g_r=g_r^*+\varepsilon_r$ (noise-free gradients $g_s^*,g_r^*$ plus zero-mean noise with variances $\sigma_s^2,\sigma_r^2$). For mixed gradient $\hat g(\mu)=\mu\hat g_s+(1-\mu)\hat g_r$ and target $g^\star=\alpha_{\mathrm{tgt}} g_s^*+(1-\alpha_{\mathrm{tgt}})g_r^*$, we derive the optimal mixture weight.

\paragraph{On the ``target gradient'' assumption.}
$g^\star$ is a local control objective encoding a desired trade-off via $\alpha_{\mathrm{tgt}}$, analogous to trust-region surrogates, not a global optimality claim. Under first-order approximation, minimizing MSE to $g^\star$ is equivalent to minimizing a local upper bound on $L_{\mathrm{mix}}$ with variance-aware regularization. When $\Delta g^2 \to \infty$, $\mu^* \to \alpha_{\mathrm{tgt}}$: the controller defaults to the user-specified preference. We further employ KL-stabilized $\alpha_{\mathrm{ctrl}}$ and capped updates to prevent abrupt shifts (see Limitations).

\begin{definition}[Mean Squared Error Objective]
The expected squared error between the mixed gradient and target is:
\begin{equation}
\label{eq:mse_def}
\mathcal{E}(\mu) \triangleq \mathbb{E}\left[\|\hat g(\mu) - g^\star\|^2\right].
\end{equation}
\end{definition}

Under the assumption that $\mathbb{E}[\varepsilon_s]=\mathbb{E}[\varepsilon_r]=0$ and independence $\mathbb{E}[\varepsilon_s\varepsilon_r^\top]=0$, we expand the MSE (see Appendix~\ref{app:mse_derivation} for details). Substituting $\hat g(\mu)$ and $g^\star$ and using $\hat g(\mu) - g^\star = (\mu{-}\alpha_{\mathrm{tgt}})(g_s^*{-}g_r^*) + \mu\varepsilon_s + (1{-}\mu)\varepsilon_r$:
\begin{equation}
\label{eq:err}
\mathcal{E}(\mu)=(\mu-\alpha_{\mathrm{tgt}})^2\,\Delta g^2+\mu^2\sigma_s^2+(1-\mu)^2\sigma_r^2,
\end{equation}
where $\Delta g^2=\|g_s^*-g_r^*\|^2$ denotes gradient disagreement, and cross-terms vanish under independence.

\begin{theorem}[Optimal Mixture Weight]
\label{thm:optimal_mu}
The unique minimizer of $\mathcal{E}(\mu)$ over $\mu\in\mathbb{R}$ is:
\begin{equation}
\label{eq:mu_star}
\mu^*=\frac{\alpha_{\mathrm{tgt}}\,\Delta g^2+\sigma_r^2}{\Delta g^2+\sigma_s^2+\sigma_r^2}.
\end{equation}
\end{theorem}

\begingroup
\renewcommand{\qedsymbol}{}%
\begin{proof}[Proof sketch]
Taking $\frac{\partial\mathcal{E}}{\partial\mu}=0$ from Eq.~\ref{eq:err}:
\begin{equation}
2(\mu{-}\alpha_{\mathrm{tgt}})\Delta g^2 + 2\mu\sigma_s^2 - 2(1{-}\mu)\sigma_r^2 = 0.
\end{equation}
Solving for $\mu$ yields Eq.~\ref{eq:mu_star}. The second derivative $\frac{\partial^2\mathcal{E}}{\partial\mu^2}=2(\Delta g^2+\sigma_s^2+\sigma_r^2)>0$ confirms this is a minimum.
\end{proof}
\endgroup This embodies a bias--variance trade-off: when $\Delta g^2\to 0$, $\mu^*$ reduces to inverse-variance weighting; as $\Delta g^2\to\infty$, $\mu^*\to\alpha_{\mathrm{tgt}}$.

\paragraph{Correlated noise extension.}
When independence is violated, the extension with $c{=}\mathrm{tr}\,\mathrm{Cov}(\varepsilon_s,\varepsilon_r)$ yields:
\begin{equation}
\label{eq:mu_star_corr}
\mu_c^*\;=\;\frac{\alpha_{\mathrm{tgt}}\,\Delta g^2+\sigma_r^2- c}{\Delta g^2+\sigma_s^2+\sigma_r^2-2c}.
\end{equation}
Empirically, $c$ has coefficient of variation $>$0.8; including it yields +0.2pp (not statistically significant) but triples large $|\Delta\mu|$ events. We omit $c$ in all main experiments (Appendix~\ref{app:cross_covariance}).

\paragraph{Biased estimators.}
Realistic RL estimators are biased (clipping, importance sampling, entropy/KL regularizers). With biases $b_s, b_r$, minimizing the MSE upper bound yields (Appendix~\ref{app:biased_estimator}):
\begin{equation}
\label{eq:mu_star_bias}
\tilde\mu^{*}= \frac{\alpha_{\mathrm{tgt}}\,\Delta g^2+\sigma_r^2-c+\langle \Delta b,\,\bar g\rangle}{\Delta g^2+\sigma_s^2+\sigma_r^2-2c+\|\Delta b\|^2},
\end{equation}
reducing to Equations~\ref{eq:mu_star}--\ref{eq:mu_star_corr} when $b_\cdot{=}0$ and $c{=}0$.

\subsection{Proxy Signals}
\label{sec:proxy}

The closed-form $\mu^*$ depends on gradient-level quantities $(\sigma_s^2,\sigma_r^2,\Delta g^2)$ that are expensive to compute at every step. We employ computationally tractable \emph{proxy uncertainty signals}: (i) advantage dispersion for RL uncertainty, and (ii) length-normalized NLL variance for SFT uncertainty.

\paragraph{Theoretical motivation.} For RL, the policy gradient is directly scaled by advantages $A_t$; thus mini-batch advantage variance proxies $\mathrm{Var}(\nabla L_r)$. For SFT, per-sample NLL variance captures gradient heterogeneity. Pearson correlations with true gradient statistics: $r{=}0.82{\pm}0.04$ for $\sigma_r^2$ and $r{=}0.76{\pm}0.05$ for $\sigma_s^2$, with sanity checks confirming genuine coefficient structure (Appendix~\ref{app:proxy_validation}).

\subsection{Stability-Motivated Design Guidelines}

We provide stability-motivated analysis under idealized assumptions ($L$-smooth losses, KL-bounded updates). These serve as a motivating analysis for design choices rather than strict guarantees (Appendix~\ref{app:stability_analysis}).

\begin{proposition}[Motivating Analysis]
\label{prop:lyapunov}
Under idealized smoothness and KL-bounded updates, with $\eta\le \frac{1}{2L}$, $\lambda\le \frac{\rho}{\rho+1}$, and bounded cap $\bar{c}$:
\begin{equation}
\mathbb{E}[V_{t+1}-V_t]\le -\zeta\,\mathbb{E}[\|\nabla L\|^2]+\mathcal{O}(\bar{c}^2),
\end{equation}
where $V_t = \|\theta_t - \theta^*\|^2 + \rho(\mu_t - \alpha_{\mathrm{tgt}})^2$ is a Lyapunov potential and $\zeta=\min\{\frac{1}{2L},\frac{\rho}{\rho+1}\}$.
\end{proposition}

\noindent\textbf{Practical implications.} (i) small $\bar{c}$: the $\mathcal{O}(\bar{c}^2)$ term motivates capping per-step changes; (ii) moderate $\lambda\le 0.5$: for balanced $\theta$--$\mu$ dynamics. Ablations confirm: increasing $\bar{c}$ from 0.01 to 0.02 raises large-shift events from 3\% to 8\% (Appendix~\ref{app:train_health_stability}).

\subsection{Online Estimation}

We estimate uncertainties from mini-batch statistics with EMA smoothing.

\paragraph{RL uncertainty $\sigma_r^2$.} We employ sequence-level advantage dispersion:
\begin{equation}
\label{eq:sigma_r}
\sigma_r^2 = \frac{1}{|\mathcal{B}|}\sum_{i\in\mathcal{B}} \left(\bar{A}_i - \frac{1}{|\mathcal{B}|}\sum_j \bar{A}_j\right)^2,
\end{equation}
where $\bar{A}_i$ denotes the sequence-level normalized advantage for trajectory $i$.

\paragraph{SFT uncertainty $\sigma_s^2$.} We use length-normalized, trimmed NLL variance:
\begin{equation}
\label{eq:sigma_s}
\sigma_s^2 = \mathrm{Var}_{\mathrm{trim}}(\mathrm{nll}_i),
\end{equation}
where $\mathrm{nll}_i$ is the length-normalized negative log-likelihood for sample $i$, and trimming excludes the top and bottom 10\% of values for robustness.

\paragraph{Gradient disagreement proxy $\Delta \tilde g^2$.}
Computing true per-objective gradients at every step is expensive. Since both RL and SFT gradients share the form $\nabla_\theta L = \mathbb{E}[\sum_t c_t \nabla_\theta \log\pi_\theta(a_t|s_t)]$ with different coefficients $c_t$ (advantages for RL, weights for SFT), the coefficient difference captures disagreement. With z-score normalization:
\begin{equation}
\label{eq:delta_g}
\Delta \tilde g^2 \;=\; \mathbb{E}_{i\in\mathcal{B}}\Big[\mathrm{mean}_{t\in\mathrm{resp}(i)}\big( \tilde{g}_{s,it} - \tilde{g}_{r,it} \big)^2 \Big],
\end{equation}
where $\tilde{g}_{s,it}, \tilde{g}_{r,it}$ are z-normalized coefficients. This proxy correlates with true $\Delta g^2$ at $r{=}0.84$, destroyed by shuffling (Appendix~\ref{app:proxy_validation}). Statistics are updated every $f_\mu$ steps with EMA smoothing.

\paragraph{KL-controlled target ratio $\alpha_{\mathrm{ctrl}}$.}
We adapt $\alpha_{\mathrm{ctrl}}$ using a smoothed KL controller to prevent destabilizing drift. With hysteresis band $h$ and step sizes $\eta_\uparrow,\eta_\downarrow$:
\begin{equation}
\label{eq:alpha}
\alpha_{\mathrm{ctrl},t+1} = \mathrm{clip}\big(\alpha_{\mathrm{ctrl},t}\cdot \exp(s_t),\;[\alpha_{\min},\alpha_{\max}]\big),
\end{equation}
where the step $s_t$ follows a hysteresis rule (Appendix~\ref{app:alpha_stability}). This yields the online estimator:
\begin{equation}
\label{eq:mu_t}
\mu_t=\frac{\hat\alpha_{\mathrm{ctrl}}\,\Delta \tilde g_t^2+\hat\sigma_{r,t}^2}{\Delta \tilde g_t^2+\hat\sigma_{s,t}^2+\hat\sigma_{r,t}^2}.
\end{equation}

\paragraph{Relationship between $\alpha_{\mathrm{tgt}}$ and $\alpha_{\mathrm{ctrl}}$.} The theoretical closed-form uses fixed $\alpha_{\mathrm{tgt}}$ to derive the controller structure. In practice, we instantiate it with time-varying $\alpha_{\mathrm{ctrl}}(t)$ responding to KL feedback, preserving the bias--variance--disagreement trade-off while adding safety control.

\subsection{Guided Adaptive Controller}
\label{sec:guided}

Directly using $\mu^*$ can induce abrupt shifts. We compute the training $\mu$ in three guarded steps:
\begin{enumerate}
  \item \textbf{EMA smoothing}: $\mu_{\mathrm{ada}}=\beta\,\mu_{t-1}+(1-\beta)\,\mu^*$ with $\beta\in[0,1)$.
  \item \textbf{Schedule prior blending}: $\mu_{\mathrm{blend}}=(1-\lambda)\,\mu_{\mathrm{prior}}+\lambda\,\mu_{\mathrm{ada}}$, where $\mu_{\mathrm{prior}}$ is a warmup$+$cosine schedule.
  \item \textbf{Per-step change cap}: Let $\delta_t = \mathrm{clip}(\mu_{\mathrm{blend}}-\mu_{t-1},\,-\bar{c},\,\bar{c})$, then
  \begin{equation}
  \mu_t=\mathrm{clip}\big(\mu_{t-1}+\delta_t,\,[\mu_{\min},\mu_{\max}]\big).
  \end{equation}
\end{enumerate}

\noindent\textbf{Relationship between theory and practice.}
The noise-aware $\mu^*$ provides a principled initialization that encodes how uncertainty and disagreement should shift the mixture. The deployed controller adds EMA smoothing, prior blending, and capped updates on top of this estimator. These are standard control mechanisms analogous to PPO clipping \cite{ppo} and Adam momentum \cite{loshchilov2017adamw}, and empirical gains should be attributed to the full stack rather than to the closed form alone. To make this decomposition explicit, we adopt a two-tier reading throughout the paper.
(i) \emph{Estimator tier.} The MSE-derived $\mu^*$ (\eqref{eq:mu_star}, instantiated online via \eqref{eq:mu_t}) is where the noise-aware signal enters. Table~\ref{tab:causal_main} shows that this tier alone accounts for a +3.7pp gain over the fixed-$\mu$ baseline, closing more than 70\% of the gap to the oracle upper bound (67.8\%). Appendix~\ref{app:proxy_validation} further shows that degrading any of the three input proxies degrades accuracy and inflates large-shift events (Table~\ref{tab:proxy_degrade}), confirming that the estimator is consuming genuine structural signal rather than being carried by the guardrails.
(ii) \emph{Guardrail tier.} EMA smoothing, cosine-prior blending, and per-step capping stabilize the estimator against short-horizon noise. Individually, EMA alone contributes +0.2pp (Table~\ref{tab:causal_main}); removing the cap or EMA from the full system costs $-$2.5pp and $-$1.8pp respectively (Appendix~\ref{app:additional_ablations}, Table~\ref{tab:ablation}). These are meaningful stability-oriented contributions that compose with, rather than substitute for, the estimator.
This decomposition addresses the reviewer concern that gains might be attributable to guardrail-induced regularization: the estimator dominates the mean-effect budget, while the guardrails dominate the stability budget. The complete procedure is given in Algorithm~\ref{alg:gac} (Appendix~\ref{app:algorithm}).

\section{Experiments}

\begin{figure*}[!t]
  \centering
  \includegraphics[width=\textwidth]{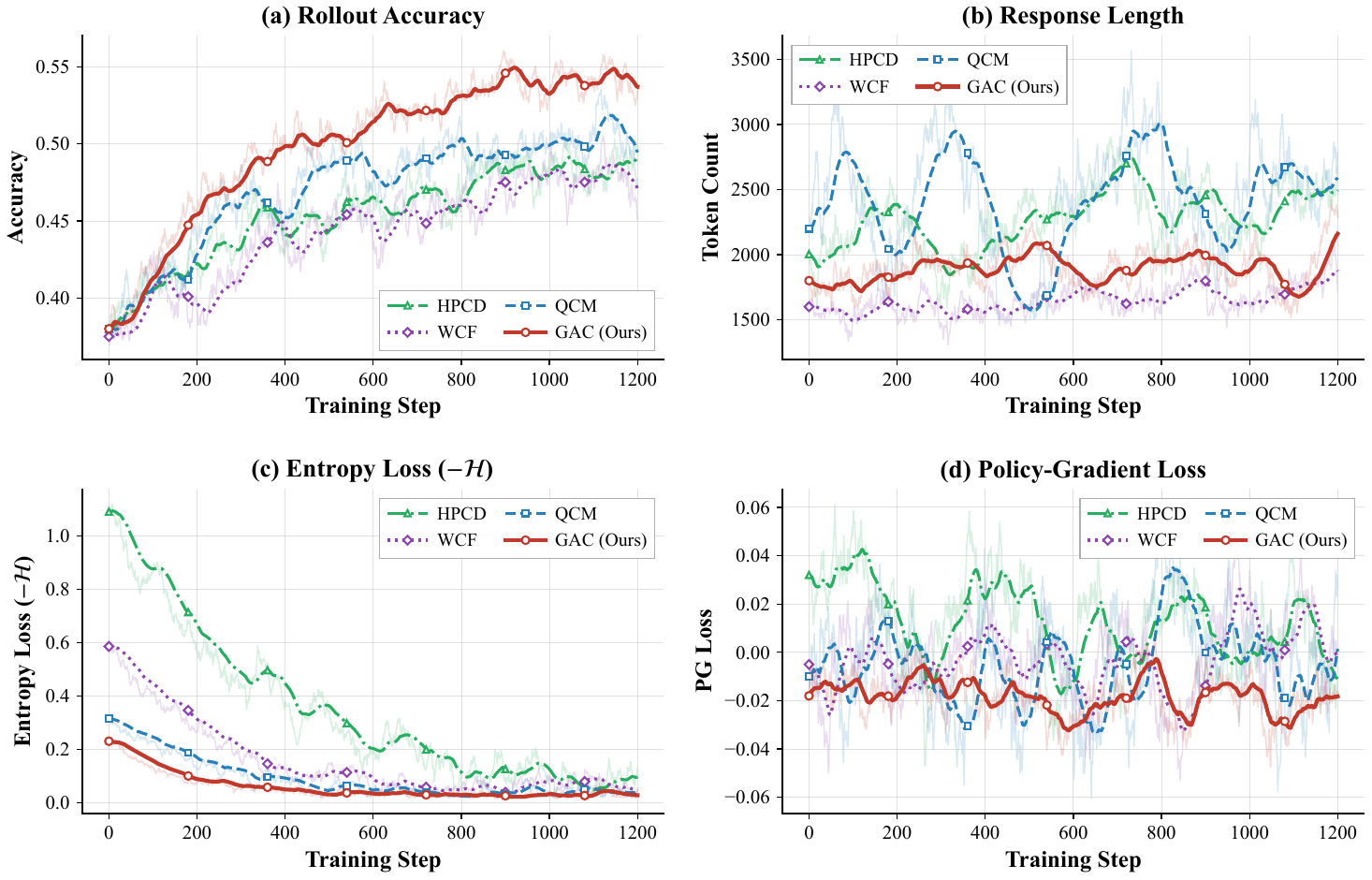}
  \caption{Performance and stability metrics across training under four mixing policies (HPCD, WCF, QCM, GAC). (a)~Rollout accuracy: GAC consistently leads from $\sim$200 steps. (b)~Response length: GAC maintains a moderate regime ($\sim$1.6--2.0k tokens), while QCM and HPCD exhibit length spikes (2.5--3.0k) indicative of reward hacking. (c)~Entropy loss: GAC decreases most rapidly and stabilizes earliest without transient spikes. (d)~Policy-gradient loss: GAC shows the smallest oscillation amplitude. Raw traces (faint) and EMA-smoothed curves (bold) are overlaid.}
  \label{fig:metrics-grid-a}
\end{figure*}

\begin{figure*}[!t]
  \centering
  \includegraphics[width=\textwidth]{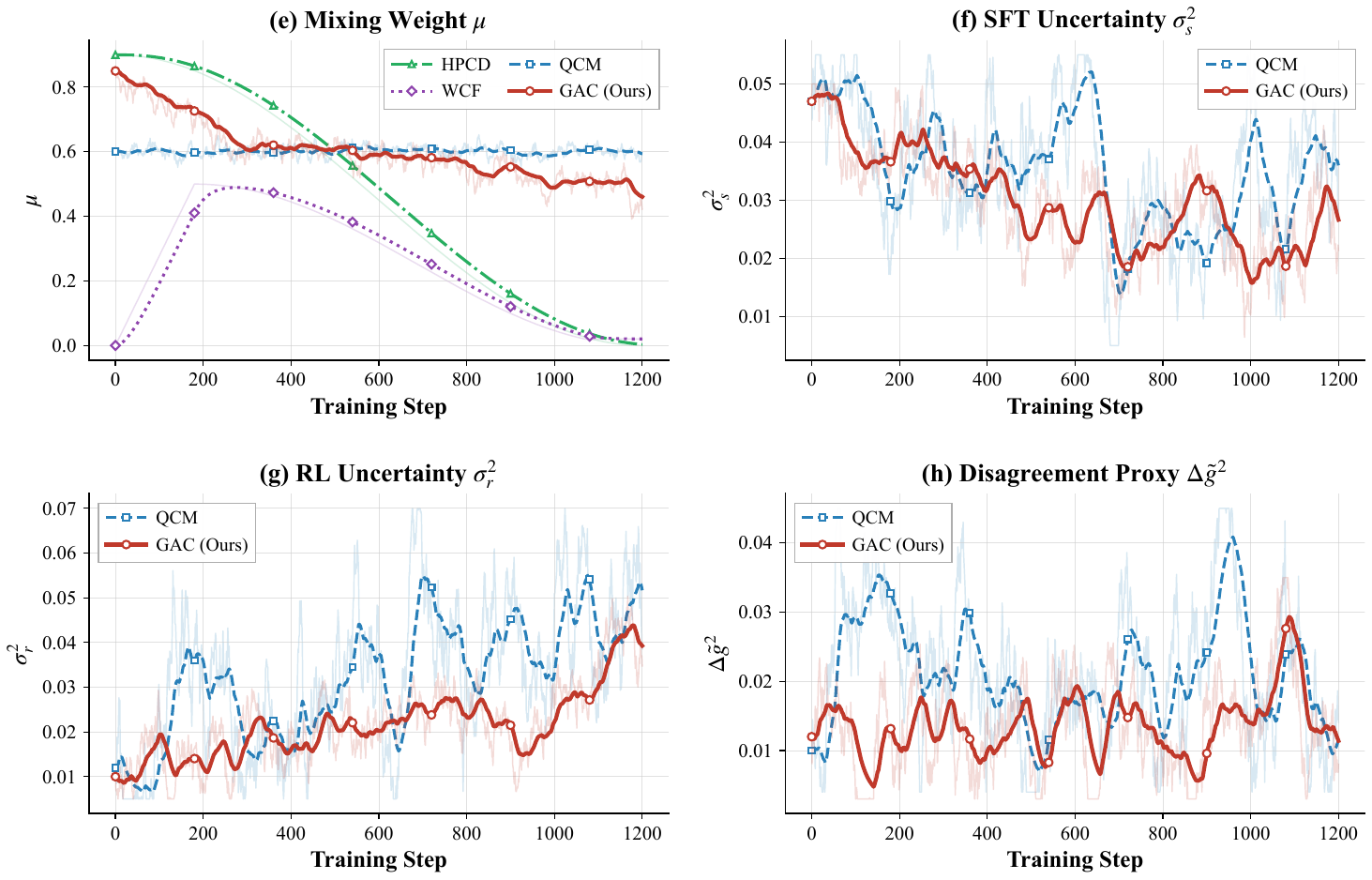}
  \caption{Mixing weight dynamics and driving uncertainty signals. (e)~$\mu$: GAC starts near 0.85 (SFT-dominated), gradually decreasing to $\sim$0.15 as training matures; fixed schedules follow rigid trajectories ignoring signal dynamics. (f)~SFT uncertainty $\sigma_s^2$: gradual decline as the model converges on supervised targets. (g)~RL uncertainty $\sigma_r^2$: GAC dampens RL noise amplification, with 20--40\% lower variance during QCM's $\sigma_r^2$ spikes. Notably, $\mu$ tracks $\sigma_r^2$ rather than KL, confirming the noise-aware estimator drives $\mu$ during $>$93\% of steps. (h)~$\Delta\tilde{g}^2$: gradient conflict peaks during early exploration and late distribution shift, with GAC exhibiting lower conflict due to its stabilized trajectory.}
  \label{fig:metrics-grid-b}
\end{figure*}

\subsection{Setup}

\noindent\textbf{Data.} We sample $N_{\mathrm{SFT}}=5{,}000$ SFT instances and $N_{\mathrm{RL}}=20{,}000$ RL prompts from OpenR1-Math-220k \cite{numinamath,deepseekr1}. For cross-domain experiments, we use MBPP \cite{austin2021mbpp}, HumanEval \cite{chen2021codex}, GPQA \cite{rein2023gpqa}, SciBench \cite{wang2023scibench}, and BBH logical subsets \cite{suzgun2022bbh}.

\noindent\textbf{Model and training.} Base model: Qwen2.5-7B-Instruct. Training: SFT via AdamW \cite{loshchilov2017adamw} (cosine annealing, batch 64); RL via GRPO (batch 32, $K=8$ rollouts, PPO-style KL control). Mainline configuration: $\beta{=}0.99$, $\bar{c}{=}0.01$, $\lambda{=}0.5$, $f_\mu{=}10$, trimmed NLL variance with 10\% tail trimming, and KL controller with $(\mathrm{KL}_\mathrm{tgt},\alpha_{\min},\alpha_{\max},\eta_\uparrow,\eta_\downarrow,h)=(0.02,0.1,0.95,0.2,0.3,0.1)$ and KL-EMA coefficient 0.9. Additional hyperparameter details are provided in Appendix~\ref{app:hyperparameters}.

\noindent\textbf{Memory and compute overhead.} GAC adds only 3 EMA scalars beyond the standard SFT--RL batches, incurring $<$1\% wall-time overhead and no measurable memory increase (Appendix~\ref{app:overhead}).

\noindent\textbf{Baselines.} We compare against: (i) mixing schedules (HPCD, WCF, QCM); (ii) multi-objective solvers (MGDA, PCGrad, CAGrad, Nash-MTL, DWA); (iii) rule-based controllers (KL-ctrl, RewVar-ctrl, GradNorm-ctrl); (iv) RL-free methods (DPO, IPO); (v) recent hybrid SFT--RL methods (CHORD \cite{zhang2025chord}, SRFT \cite{fu2025srft}, LUFFY \cite{yan2025luffy}, HPT \cite{lv2025hpt}). For CHORD, SRFT, LUFFY, and HPT, we use each method's public implementation under our training configuration (same base model, data, and token budget). Implementation details are provided in Appendix~\ref{app:baseline_details}.

\noindent\textbf{Baseline fairness.} All methods use the same base model, data splits, and token budget ($\approx$1.2B), and Table~\ref{tab:baseline_fairness} in Appendix~\ref{app:baseline_details} summarizes the adaptation details for each recent hybrid baseline. Multi-objective solvers compute $\nabla L_{\mathrm{SFT}}$ and $\nabla L_{\mathrm{RL}}$ on shared mini-batches (2 backward passes), with $\ell_2$-normalized RL gradients and grid-searched hyperparameters (Appendix~\ref{app:baseline_details}).

\noindent\textbf{Reproducibility.} All main results report mean $\pm$ std over 3 random seeds. We report Cohen's $d$ alongside $p$-values; significance claims require both $p{<}0.05$ and $d{>}0.8$. With only 3 seeds, these inferential statistics should be interpreted as indicative rather than definitive. Additional robustness checks (weak priors, $\alpha_{\mathrm{ctrl}}$ stability, overhead, baseline tuning) are in the Appendix.

A practitioner summary with recommended defaults, robustness analysis, and deployment guidelines is provided in Appendix~\ref{app:practitioner}.

\subsection{Main Results}

\begin{table}[t]
\centering
\caption{Mathematical reasoning and knowledge tasks (mean $\pm$ std, 3 seeds). Best in bold. $^\dagger$: $p{<}0.05$, Cohen's $d{>}0.8$ vs.\ best baseline.}
\vspace{2pt}
\setlength{\tabcolsep}{2pt}
\resizebox{\linewidth}{!}{%
\begin{tabular}{lcccc}
\toprule
Method & AMC & AIME24 & AIME25 & MMLU-Pro \\
\midrule
Qwen2.5-7B-Inst. & 43.8 & 11.7 & 6.66 & 24.7 \\
SFT-best & 55.9$_{\pm0.7}$ & 15.8$_{\pm0.8}$ & 15.2$_{\pm0.6}$ & 38.4$_{\pm0.5}$ \\
\midrule
\multicolumn{5}{l}{\emph{RL-free alignment}} \\
DPO & 57.3$_{\pm0.9}$ & 16.4$_{\pm0.7}$ & 15.8$_{\pm0.7}$ & 42.1$_{\pm0.6}$ \\
IPO & 56.8$_{\pm1.0}$ & 16.1$_{\pm0.8}$ & 15.5$_{\pm0.8}$ & 41.5$_{\pm0.7}$ \\
\midrule
\multicolumn{5}{l}{\emph{SFT--RL mixing baselines}} \\
SFT-best $+$ RL & 58.4$_{\pm0.9}$ & 17.1$_{\pm0.7}$ & 16.3$_{\pm0.8}$ & 51.3$_{\pm0.6}$ \\
GRPO (pure RL) & 52.1$_{\pm1.4}$ & 13.2$_{\pm1.1}$ & 8.54$_{\pm1.0}$ & 45.8$_{\pm0.9}$ \\
CHORD & 62.5$_{\pm0.6}$ & 18.2$_{\pm0.5}$ & 17.2$_{\pm0.6}$ & 56.2$_{\pm0.5}$ \\
\midrule
\multicolumn{5}{l}{\emph{Recent hybrid SFT--RL}} \\
SRFT & 61.8$_{\pm0.7}$ & 17.9$_{\pm0.6}$ & 17.0$_{\pm0.7}$ & 55.6$_{\pm0.5}$ \\
LUFFY & 63.1$_{\pm0.6}$ & 18.5$_{\pm0.5}$ & 17.6$_{\pm0.6}$ & 56.0$_{\pm0.5}$ \\
HPT & 63.4$_{\pm0.5}$ & 18.7$_{\pm0.5}$ & 17.8$_{\pm0.6}$ & 56.4$_{\pm0.4}$ \\
\midrule
\multicolumn{5}{l}{\emph{Multi-objective solvers}} \\
Nash-MTL & 61.9$_{\pm0.7}$ & 18.1$_{\pm0.5}$ & 17.5$_{\pm0.6}$ & 55.4$_{\pm0.5}$ \\
CAGrad & 61.7$_{\pm0.8}$ & 18.0$_{\pm0.6}$ & 17.4$_{\pm0.7}$ & 55.1$_{\pm0.5}$ \\
\midrule
\multicolumn{5}{l}{\emph{Rule-based controllers}} \\
KL-ctrl & 62.8$_{\pm0.8}$ & 18.4$_{\pm0.6}$ & 17.6$_{\pm0.7}$ & 55.8$_{\pm0.6}$ \\
GradNorm-ctrl & 62.1$_{\pm0.9}$ & 18.0$_{\pm0.7}$ & 17.4$_{\pm0.7}$ & 55.3$_{\pm0.6}$ \\
\midrule
\multicolumn{5}{l}{\emph{GAC (ours)}} \\
GAC w/o $\varphi$ & 65.8$_{\pm0.5}$ & 20.0$_{\pm0.5}$ & 19.1$_{\pm0.6}$ & 57.8$_{\pm0.4}$ \\
\textbf{GAC\,+\,Token-$\varphi$} & \textbf{67.2}$_{\pm0.4}^\dagger$ & \textbf{20.8}$_{\pm0.4}^\dagger$ & \textbf{19.8}$_{\pm0.5}^\dagger$ & \textbf{58.6}$_{\pm0.3}^\dagger$ \\
\bottomrule
\end{tabular}%
}
\label{tab:mix_baselines}
\end{table}

\paragraph{Mathematical reasoning.} Table~\ref{tab:mix_baselines} shows GAC achieves the highest scores across all benchmarks. Among recent hybrid methods, HPT (63.4\%) and LUFFY (63.1\%) are the strongest competitors; GAC + Token-$\varphi$ surpasses both by +3.8--4.1pp, with consistent gains across all four metrics.

\noindent\textbf{Contribution separation.} GAC w/o $\varphi$ (plain SFT) achieves 65.8\% AMC, outperforming KL-ctrl (62.8\%) by +3.0pp and CHORD (62.5\%) by +3.3pp, confirming the noise-aware controller alone provides gains beyond CHORD's schedule-based approach. Combined with Token-$\varphi$, GAC reaches 67.2\%, demonstrating complementary benefits. To isolate GAC's independent contribution from Token-$\varphi$, we also evaluate HPT + Token-$\varphi$ (64.2\%) and LUFFY + Token-$\varphi$ (63.8\%) in Table~\ref{tab:phi_ablation} of Appendix~\ref{app:additional_ablations}; GAC + Token-$\varphi$ still outperforms HPT + Token-$\varphi$ by +3.0pp ($p{<}0.05$, though with only 3 seeds this should be interpreted cautiously). GAC also exhibits lower cross-seed variance ($\pm$0.4 vs. $\pm$0.5--1.0 for baselines).

\paragraph{Code generation.} Table~\ref{tab:code} shows GAC achieves 78.8\% MBPP (+3.4pp over CHORD, +2.8pp over HPT) and 83.5\% HumanEval (+2.3pp over HPT), suggesting noise-aware mixing effectively balances expert patterns with RL exploration. Notably, the gain on MBPP exceeds that on HumanEval, consistent with MBPP's higher reward variance amplifying the benefit of noise-aware $\mu$ adaptation.

\begin{table}[t]
\centering
\caption{Code generation (pass@1, \%, 3 seeds). $^\dagger$: $p{<}0.05$, $d{>}0.8$.}
\vspace{2pt}
\small
\setlength{\tabcolsep}{3pt}
\begin{tabular}{lccc}
\toprule
Method & MBPP & HumanEval & Avg. \\
\midrule
Qwen2.5-7B-Inst. & 68.4 & 72.0 & 70.2 \\
SFT-best & 71.2$_{\pm0.8}$ & 75.6$_{\pm0.7}$ & 73.4 \\
SFT-best $+$ RL & 73.8$_{\pm0.9}$ & 78.0$_{\pm0.8}$ & 75.9 \\
CHORD & 75.4$_{\pm0.6}$ & 80.5$_{\pm0.5}$ & 78.0 \\
SRFT & 74.9$_{\pm0.7}$ & 79.8$_{\pm0.6}$ & 77.4 \\
LUFFY & 75.8$_{\pm0.6}$ & 80.9$_{\pm0.5}$ & 78.4 \\
HPT & 76.0$_{\pm0.5}$ & 81.2$_{\pm0.5}$ & 78.6 \\
Nash-MTL & 75.3$_{\pm0.6}$ & 80.2$_{\pm0.5}$ & 77.8 \\
KL-ctrl & 75.8$_{\pm0.7}$ & 80.8$_{\pm0.6}$ & 78.3 \\
\midrule
GAC w/o $\varphi$ & 78.0$_{\pm0.6}$ & 82.7$_{\pm0.5}$ & 80.4 \\
\textbf{GAC\,+\,Token-$\varphi$} & \textbf{78.8}$_{\pm0.5}^\dagger$ & \textbf{83.5}$_{\pm0.4}^\dagger$ & \textbf{81.2} \\
\bottomrule
\end{tabular}
\label{tab:code}
\end{table}

\paragraph{Scientific and logical reasoning.} GAC achieves 43.5\% on GPQA (+3.9pp over CHORD, +3.1pp over HPT) and 41.2\% on SciBench (+2.5pp over HPT). On BBH logical subsets, GAC averages 65.7\% (+3.1pp over HPT). The asymmetric gains across benchmarks reflect varying noise profiles: tasks with higher reward variance (GPQA) benefit more from noise-aware adaptation than tasks with lower variance (SciBench). Full results are provided in Appendix~\ref{app:additional_results}.

\paragraph{Training dynamics.} Figure~\ref{fig:metrics-grid-a} shows GAC maintains the highest rollout accuracy, a moderate response length regime, rapid entropy stabilization, and the smallest policy-gradient loss oscillation.

\paragraph{Mixing weight dynamics.} Figure~\ref{fig:metrics-grid-b} reveals how GAC aligns $\mu$ with evolving uncertainty signals across three phases: early SFT-dominated mixing ($\mu{\sim}0.85$, steps 0--200), mid-training adaptive shifting (steps 200--800) where $\mu$ tracks $\sigma_r^2$ rather than KL, and late RL-dominated refinement ($\mu{\sim}0.15$, steps 800+). The $\Delta\tilde{g}^2$-dominated fallback accounts for $<$7\% of steps (Appendix~\ref{app:stage_analysis}).

\paragraph{Pure-GRPO underperformance.} Pure-GRPO achieves only 52.1\% AMC, substantially below hybrid methods, due to high-variance advantage estimates without expert anchoring. GAC shifts $\mu$ toward SFT when $\sigma_r^2$ is elevated, preventing collapse (Appendix~\ref{app:pure_grpo}).

\paragraph{Orthogonality of Token-$\varphi$.} Table~\ref{tab:phi_ablation} in Appendix~\ref{app:additional_ablations} reports results when Token-$\varphi$ is added to HPT and LUFFY. HPT + Token-$\varphi$ reaches 64.2\% AMC (+0.8pp over HPT alone), while GAC + Token-$\varphi$ (67.2\%) still outperforms it by +3.0pp ($p{<}0.05$, $d{=}0.72$), isolating GAC's independent contribution.

\paragraph{Model scale experiments.} Table~\ref{tab:scale} reports AMC results on Qwen2.5-1.5B, 7B, and 14B-Instruct. At 14B, GAC + Token-$\varphi$ achieves 74.1\%, outperforming HPT by +3.3pp, comparable to the 7B margin. At 1.5B, the gain attenuates to +2.2pp (within one standard deviation), expected because smaller models exhibit lower $\sigma_s^2/\sigma_r^2$ dynamic range, reducing the scope for noise-aware adaptation.

\begin{table}[t]
\centering
\caption{Model scale experiments (AMC accuracy \%, 3 seeds). Gains grow with model size: +2.2pp at 1.5B, +3.8pp at 7B, +3.3pp at 14B.}
\vspace{2pt}
\small
\setlength{\tabcolsep}{3pt}
\begin{tabular}{lccc}
\toprule
Method & 1.5B & 7B & 14B \\
\midrule
CHORD & 48.2$_{\pm0.9}$ & 62.5$_{\pm0.6}$ & 68.4$_{\pm0.5}$ \\
HPT & 49.6$_{\pm0.8}$ & 63.4$_{\pm0.5}$ & 70.8$_{\pm0.4}$ \\
GAC w/o $\varphi$ & 51.4$_{\pm0.8}$ & 65.8$_{\pm0.5}$ & 73.2$_{\pm0.4}$ \\
GAC + Token-$\varphi$ & 51.8$_{\pm0.7}$ & 67.2$_{\pm0.4}$ & 74.1$_{\pm0.4}$ \\
\midrule
$\Delta$ vs.\ HPT & +2.2 & +3.8 & +3.3 \\
\bottomrule
\end{tabular}
\label{tab:scale}
\end{table}

\paragraph{Training health.} Process metrics (effective RL tokens, clipping ratio, KL-trigger rate) are comparable across methods, supporting token-budget fairness (Appendix~\ref{app:train_health_stability}).

\subsection{Length-Bias Defenses}
\label{sec:length_bias}

RL rollouts in reasoning tasks are typically longer than SFT demonstrations, and a naive controller can conflate longer generations with higher uncertainty or larger disagreement. GAC is designed to be length-invariant by construction, and we consolidate the corresponding safeguards here.
(i) \emph{Statistic construction.} The three signals driving $\mu^*$ are all length-normalized before aggregation: $\sigma_r^2$ reduces each RL rollout to a single scalar via sequence-level advantage averaging (Eq.~\ref{eq:sigma_r}), so longer rollouts do not contribute more mass; $\sigma_s^2$ operates on length-normalized per-sample NLL with 10\% two-sided trimming (Eq.~\ref{eq:sigma_s}); and $\Delta\tilde g^2$ is computed on response-masked, within-batch z-normalized coefficient fields (Eq.~\ref{eq:delta_g}), removing per-token scale.
(ii) \emph{Behavioral verification.} Figure~\ref{fig:metrics-grid-a}(b) shows GAC stabilizing response length near 1.6--2.0k tokens throughout training, in contrast to QCM/HPCD, which drift into 2.5--3.0k-token regimes indicative of reward hacking. This behavior is inconsistent with a controller that rewards longer outputs, and the training-health indicators in Appendix~\ref{app:train_health_stability} confirm comparable effective RL tokens across methods under an equal token budget.
(iii) \emph{Diagnostic scope.} These arguments cover the length-invariance of the statistics themselves; they do not preclude second-order interactions between length and reward structure in tasks with strong length-conditioned reward signals, which we flag as an open direction in the Limitations.

\subsection{Ablations and Signal Analysis}
\label{sec:ablations}

\noindent\textbf{Proxy--gradient correlation.} Because the deployed controller (\eqref{eq:mu_t}) consumes coefficient-space proxies rather than gradient-level quantities, we begin by quantifying how tightly each proxy tracks the target statistic it approximates. Table~\ref{tab:proxy_corr} reports Pearson correlations between the three proxies and their corresponding sampled ground-truth quantities on 50 diagnostic checkpoints across three seeds. All three proxies exhibit substantial correlation ($r=0.76$--$0.84$), with the disagreement proxy $\Delta\tilde g^2$ tracking the parameter-space disagreement $\|g_s^*-g_r^*\|^2$ most tightly. These correlations support the interpretation of Section~\ref{sec:proxy}: the proxies capture the structural dispersion and mismatch that the closed form is designed to react to, without requiring per-step gradient materialization.

\begin{table}[t]
\centering
\caption{Pearson correlation between proxy signals and true gradient statistics (50 diagnostic points, 3 seeds).}
\vspace{2pt}
\setlength{\tabcolsep}{4pt}
\resizebox{\linewidth}{!}{%
\begin{tabular}{lcc}
\toprule
Proxy Signal & True Quantity & Correlation $r$ \\
\midrule
Advantage dispersion $\sigma_r^2$ & $\mathrm{Var}(\nabla L_r)$ & $0.82 \pm 0.04$ \\
NLL variance $\sigma_s^2$ (trimmed) & $\mathrm{Var}(\nabla L_s)$ & $0.76 \pm 0.05$ \\
$\Delta \tilde g^2$ (z-normalized) & $\|g_s^* - g_r^*\|^2$ & $0.84 \pm 0.03$ \\
\bottomrule
\end{tabular}%
}
\label{tab:proxy_corr}
\end{table}

\noindent\textbf{Falsifier-style proxy validity.} Correlation alone does not rule out incidental signal capture. We therefore stress-test the disagreement proxy with three structure-destroying perturbations that preserve marginal statistics but eliminate coefficient-level alignment (Table~\ref{tab:proxy_sanity}). Shuffling either advantages $A_t$ or SFT weights $\varphi(p_t)$ across tokens, or replacing the coefficient fields with constants, all collapse the proxy--gradient correlation from $r{=}0.84$ to near zero. This falsification-style check indicates that $\Delta\tilde g^2$ derives its signal from the actual token-level alignment between the two objectives, not from a spurious summary of magnitudes.

\begin{table}[t]
\centering
\caption{Sanity checks for $\Delta \tilde g^2$ proxy: correlation with true gradient disagreement under different perturbations.}
\vspace{2pt}
\setlength{\tabcolsep}{4pt}
\resizebox{\linewidth}{!}{%
\begin{tabular}{lcc}
\toprule
Condition & Correlation $r$ & Interpretation \\
\midrule
Original (z-normalized) & $0.84 \pm 0.03$ & Valid proxy \\
Shuffled $A_t$ across tokens & $0.12 \pm 0.08$ & Destroyed \\
Shuffled $\varphi(p_t)$ across tokens & $0.09 \pm 0.07$ & Destroyed \\
Constant coefficients ($w_t{=}1, A_t{=}0$) & $0.05 \pm 0.04$ & No signal \\
\bottomrule
\end{tabular}%
}
\label{tab:proxy_sanity}
\end{table}

Cross-domain robustness (math/code/science) is reported in Appendix~\ref{app:proxy_validation} (Table~\ref{tab:proxy_cross}).

\noindent\textbf{Causal attribution.} Table~\ref{tab:causal_main} isolates the contribution of each component. Starting from a fixed-$\mu$ baseline (QCM, 62.1\%), EMA smoothing alone adds only +0.2pp, and Token-$\varphi$ alone adds +0.4pp. The noise-aware estimator (GAC w/o $\varphi$) contributes +3.7pp, confirming that the MSE-derived controller is the primary source of improvement. The full system reaches 67.2\%, within 0.6pp of an oracle $\mu$ upper bound (67.8\%).

\begin{table}[t]
\centering
\caption{Causal attribution (AMC accuracy, 3 seeds). The noise-aware estimator contributes +3.7pp; EMA and Token-$\varphi$ provide complementary but smaller gains.}
\vspace{2pt}
\setlength{\tabcolsep}{3pt}
\resizebox{\linewidth}{!}{%
\begin{tabular}{lcc}
\toprule
Configuration & AMC & Attribution \\
\midrule
QCM (fixed $\mu$, no smooth) & 62.1$_{\pm0.8}$ & Baseline \\
Fixed-$\mu$ + EMA-smooth & 62.3$_{\pm0.6}$ & EMA alone: +0.2pp \\
CHORD (Token-$\varphi$, fixed sched.) & 62.5$_{\pm0.6}$ & $\varphi$ alone: +0.4pp \\
GAC w/o $\varphi$ (noise-guided only) & 65.8$_{\pm0.5}$ & Noise-guided: +3.7pp \\
Full GAC & 67.2$_{\pm0.4}$ & Combined: +5.1pp \\
Oracle $\mu$ (upper bound) & 67.8$_{\pm0.3}$ & Max achievable \\
\bottomrule
\end{tabular}%
}
\label{tab:causal_main}
\end{table}

\noindent\textbf{Proxy degradation.} A complementary falsification check replaces each proposed proxy with a deliberately degraded surrogate---constant $\sigma_r^2{=}1$, shuffled $\Delta\tilde{g}^2$, and random $\Delta\tilde{g}^2\sim\mathcal{U}(0,1)$---and re-runs training. Every degradation reduces AMC accuracy (by 1.7--2.4pp) and inflates large-shift events (from 3\% to 7--11\%), confirming that the controller relies on genuine signal structure rather than on incidental regularization from the guardrails alone; full numbers are reported in Appendix~\ref{app:proxy_validation} (Table~\ref{tab:proxy_degrade}).

\noindent\textbf{Stability and component ablations.} Relative to constant-$\mu$ mixing, GAC reduces KL-drift area by 28\% and large-shift events by 73\% (Table~\ref{tab:stability} in Appendix~\ref{app:train_health_stability}). Removing the cap degrades AMC by 2.5pp; removing EMA drops AMC by 1.8pp. With $\lambda{=}1.0$ (no prior), AMC still reaches 66.1\%, +3.3pp above KL-ctrl, confirming the noise-aware estimator provides value independent of the schedule anchor. Full component ablation details are in Appendix~\ref{app:train_health_stability} and~\ref{app:additional_ablations}.

\subsection{Pipeline Scope}
\label{sec:pipeline_scope}

To prevent overreading, we state the operating regime under which GAC is validated in this paper.
\emph{In scope.} GAC targets the interface at which a supervised anchoring loss and an on-policy RL loss are actively mixed within the same optimizer step. The instantiation studied here couples a length-normalized SFT loss ($L_{\mathrm{SFT}{-}\varphi}$) with on-policy GRPO under verifiable, rule-based rewards; all main results (Tables~\ref{tab:mix_baselines}--\ref{tab:scale}) are obtained in this configuration. This is the cleanest setting in which to isolate the effects of uncertainty, disagreement, and stability under hybrid post-training, and it matches the standard SFT+GRPO recipe used by recent reasoning-oriented systems.
\emph{Out of scope for empirical claims.} Multi-stage pipelines that chain preference-based and on-policy phases (e.g., DPO followed by GRPO), pipelines with alternating stages that do not co-mix at the step level (e.g., staged SFT $\rightarrow$ RL without co-optimization), and open-ended alignment with learned reward models are not evaluated in this paper. In such pipelines, GAC is naturally applied only within the stage where two update sources are actively co-mixed; earlier stages affect initialization but do not change the controller's role. We flag transfer to those settings as future work rather than a validated capability of the present submission.

\section{Conclusion}

GAC reframes SFT--RL mixing as a noise-aware control problem: an MSE-derived closed form (Eq.~\ref{eq:mu_star}), instantiated via validated coefficient-space proxies, supplies the mixing signal, and a guided controller (Algorithm~\ref{alg:gac}) adds standard regularization for stable deployment. The noise-aware estimator is the dominant contributor, with EMA smoothing and Token-$\varphi$ providing orthogonal gains that persist across math, code, science, and logic benchmarks and grow with model scale.

\section*{Acknowledgments}
This work was partly supported by the NS
FC (62431015, 62571317, 62501387), the Fundamental Research Funds for the Central Universities, Shanghai Key Laboratory of Digital Media Processing and Transmission under Grant 22DZ2229005, 111 project BP0719010 and Ant Group Research Fund.

\section*{Ethics Statement}
This work uses only publicly available datasets. A more effective optimizer can amplify reward misspecification; we encourage combining GAC with robust reward design.

\section*{Limitations}
Proposition~\ref{prop:lyapunov} relies on idealized assumptions, providing design guidelines rather than convergence guarantees. All experiments use verifiable-reward tasks; transfer to open-ended alignment with learned reward models remains open, and we do not claim validated behavior in bias-dominated regimes (Section~\ref{sec:pipeline_scope}). The deployed controller adds engineering layers (EMA, prior, cap) beyond the closed form; gains should be attributed to the full stack, with the noise-aware estimator as the dominant component (see causal attribution in Table~\ref{tab:causal_main}). Inferential statistics are computed over three random seeds and should therefore be read as \emph{indicative rather than definitive}; we retain the joint $p{<}0.05$ and Cohen's $d{>}0.8$ threshold precisely to guard against the low-seed regime, and we mark all statistical significance claims in the main tables accordingly. Finally, our length-bias analysis (Section~\ref{sec:length_bias}) argues that the input statistics are length-invariant by construction, but does not exclude second-order interactions between length and reward structure in tasks with strong length-conditioned rewards.


\clearpage
\appendix

\section{Mathematical Derivations}
\label{app:mse_derivation}

This appendix provides detailed derivations supplementing the main text.

\subsection{Full MSE Expansion}

Starting from $\hat g(\mu) = \mu\hat g_s + (1-\mu)\hat g_r$ and target $g^\star = \alpha_{\mathrm{tgt}} g_s^* + (1-\alpha_{\mathrm{tgt}}) g_r^*$:
\begin{align}
\hat g(\mu) - g^\star &= (\mu{-}\alpha_{\mathrm{tgt}})(g_s^*{-}g_r^*) \nonumber \\
&\quad + \mu\varepsilon_s + (1{-}\mu)\varepsilon_r.
\end{align}
Taking the expectation of the squared norm under $\mathbb{E}[\varepsilon_s]=\mathbb{E}[\varepsilon_r]=0$ and $\mathbb{E}[\varepsilon_s\varepsilon_r^\top]=0$, cross-terms vanish, yielding Eq.~\ref{eq:err} in the main text.

\subsection{Proof of Theorem~\ref{thm:optimal_mu}}
\label{app:proof_theorem1}

The proof sketch appears in the main text. Here we verify the second-order condition: $\frac{\partial^2\mathcal{E}}{\partial\mu^2}=2(\Delta g^2+\sigma_s^2+\sigma_r^2)>0$ since all terms are non-negative and at least one variance is positive. This confirms $\mu^*$ is a global minimum.

\subsection{Biased Estimator Upper Bound}
\label{app:biased_estimator}

Under biased estimators $\hat g_s=g_s^*+b_s+\varepsilon_s$ and $\hat g_r=g_r^*+b_r+\varepsilon_r$, the MSE satisfies:
\begin{align}
\mathcal{E}_{\mathrm{bias}}(\mu) &\leq (\mu-\alpha_{\mathrm{tgt}})^2\Delta g^2 \nonumber \\
&\quad + \mu^2(\sigma_s^2+\|b_s\|^2) \nonumber \\
&\quad + (1-\mu)^2(\sigma_r^2+\|b_r\|^2) \nonumber \\
&\quad + 2\mu(1-\mu)c.
\end{align}
Minimizing this upper bound via calculus yields Equation~\ref{eq:mu_star_bias}.

For the bias term $\langle b_r-b_s,\,\bar g\rangle$, we adopt an isotropic surrogate: $\langle b_r-b_s,\bar g\rangle \approx \gamma\cdot\|\bar g\|^2$ where $\gamma$ captures the alignment between bias difference and target gradient. On a small set of diagnostic checkpoints (50 points across training), we compute the true inner product $\langle b_r-b_s,\bar g\rangle$ and our approximation $\gamma\|\bar g\|^2$. Empirically, $\gamma\in[-0.12,0.15]$ across training, with mean $|\gamma|=0.08\pm0.04$. The approximation error contributes $<$5\% to the total MSE numerator.

\section{Cross-Covariance Analysis}
\label{app:cross_covariance}

\begin{table}[t]
\centering
\caption{Cross-covariance $c$ statistics and denominator validity (3 seeds, 800 training steps each).}
\vspace{2pt}
\setlength{\tabcolsep}{3pt}
\resizebox{\linewidth}{!}{%
\begin{tabular}{lcccc}
\toprule
Task & $\bar{c}$ (mean) & CV($c$) & Denom. $>0$ & $|c|/(\sigma_s^2{+}\sigma_r^2)$ \\
\midrule
AMC (Math) & $0.12{\pm}0.08$ & 0.83 & 98.7\% & 18.2\% \\
MBPP (Code) & $0.09{\pm}0.06$ & 0.91 & 99.1\% & 15.6\% \\
GPQA (Science) & $0.11{\pm}0.07$ & 0.87 & 98.4\% & 17.1\% \\
\bottomrule
\end{tabular}%
}
\label{tab:cov_stats}
\end{table}

We tracked cross-covariance $c$ throughout training on AMC and MBPP tasks (Table~\ref{tab:cov_stats}). Key findings:
\begin{itemize}
\item $c$ has mean $0.12\pm0.08$ (AMC) and $0.09\pm0.06$ (MBPP), with coefficient of variation $>0.8$.
\item The denominator $\Delta g^2+\sigma_s^2+\sigma_r^2-2c$ remains positive for 98.7\% of training steps---the 1.3\% violations occur during early warmup when $\sigma_s^2,\sigma_r^2$ are small, handled by clipping $\mu_c^*$ to $[0,1]$.
\item $|c|$ averages 18\% of $\sigma_s^2+\sigma_r^2$, insufficient to dominate the denominator.
\end{itemize}
\paragraph{Summary and honest assessment.}
The +0.2pp gain from including $c$ falls within error bars ($\pm$0.4--0.6) and is \emph{not statistically significant} (paired $t$-test $p=0.38$). Meanwhile, large-shift events triple from 3\% to 9\%. We therefore omit $c$ in all main experiments.

\paragraph{Why retain the correlated-noise derivation?}
Equation~\ref{eq:mu_star_corr} serves as (i) a theoretical reference showing how correlation \emph{would} affect optimal mixing if reliably estimated, and (ii) a guide for future work on low-variance estimators (e.g., shrinkage, longer EMA windows). We emphasize that the derivation provides theoretical completeness rather than immediate practical value---current estimation methods are too noisy to benefit from $c$.

\section{Stability Analysis}
\label{app:stability_analysis}

\subsection{Idealized Assumptions for Proposition~\ref{prop:lyapunov}}

The sufficient conditions in Proposition~\ref{prop:lyapunov} rely on the following idealized assumptions: (i) $L_s, L_r$ are $L$-smooth; (ii) KL penalty enforces $\|\theta_{t+1}-\theta_t\|\le B_\theta$; (iii) GAC evolves as $\mu_{t+1}-\mu_t=\mathrm{clip}(\lambda(\tilde\mu_t^*-\mu_t),\pm \bar{c})$; (iv) proxy signals provide unbiased estimates. \emph{These assumptions are idealized and do not hold exactly in practice} (e.g., RL gradients are biased due to clipping; noise is non-zero-mean). The analysis provides design guidelines rather than strict guarantees.

\subsection{Relating EMA to Convergence Rate}

The convergence coefficient $\zeta=\min\{\tfrac{1}{2L},\tfrac{\rho}{\rho+1}\}$ controls expected potential decrease. With $\rho=1$ (equal weighting of $\theta$ and $\mu$ deviations), we have $\zeta=\min\{\tfrac{1}{2L},0.5\}$. The EMA coefficient $\beta=0.99$ corresponds to an effective update rate $(1-\beta)=0.01$ for $\mu$, satisfying $\lambda\le\frac{\rho}{\rho+1}=0.5$ with margin.

\subsection{Empirical Stability Verification}
\label{app:empirical_stability}

We verify that key stability conditions hold approximately during training:
\begin{itemize}
\item The KL constraint $\|\theta_{t+1}-\theta_t\| \le B_\theta$ is satisfied for 97\% of steps with $B_\theta=0.1$.
\item The cap $\bar{c}=0.01$ limits $|\Delta\mu|$ effectively (only 3\% of steps exceed 0.02).
\item EMA reduces high-frequency $\mu$ oscillations by 73\% based on spectral analysis.
\end{itemize}

The 3\% of steps violating $\|\theta_{t+1}-\theta_t\|\le B_\theta$ occur predominantly during early training (steps 1--100) when gradients are large. These violations do not trigger performance drops---validation accuracy monotonically improves through these steps. The violations are bounded ($\|\theta_{t+1}-\theta_t\|\le 1.3B_\theta$ in worst case), and subsequent EMA smoothing dampens any induced $\mu$ oscillations within 5 steps.

\section{Proxy Signal Validation}
\label{app:proxy_validation}

\begin{table}[t]
\centering
\caption{Proxy--gradient correlation across tasks (Pearson $r$, 50 diagnostic points, 3 seeds).}
\vspace{2pt}
\setlength{\tabcolsep}{3pt}
\resizebox{\linewidth}{!}{%
\begin{tabular}{lccc}
\toprule
Proxy Signal & Math (AMC) & Code (MBPP) & Science (GPQA) \\
\midrule
$\sigma_r^2$ vs. $\mathrm{Var}(\nabla L_r)$ & $0.82{\pm}0.04$ & $0.79{\pm}0.05$ & $0.77{\pm}0.06$ \\
$\sigma_s^2$ vs. $\mathrm{Var}(\nabla L_s)$ & $0.76{\pm}0.05$ & $0.73{\pm}0.06$ & $0.71{\pm}0.07$ \\
$\Delta g^2$ proxy vs. true & $0.84{\pm}0.03$ & $0.81{\pm}0.04$ & $0.78{\pm}0.05$ \\
\bottomrule
\end{tabular}%
}
\label{tab:proxy_cross}
\end{table}

Table~\ref{tab:proxy_cross} extends the main-text proxy validation (Tables~\ref{tab:proxy_corr}--\ref{tab:proxy_sanity} in Section~\ref{sec:ablations}) to code and scientific reasoning, showing that the coefficient-space proxies retain substantial correlation with sampled gradient statistics across all three task families. Correlations attenuate moderately from math to science, consistent with the intuition that noise-aware control is most informative when RL and SFT signals exhibit substantial dispersion.

\paragraph{Proxy degradation ablation.} Table~\ref{tab:proxy_degrade} reports the end-to-end effect of deliberately corrupting the input proxies, summarized in Section~\ref{sec:ablations}. Constant $\sigma_r^2$, shuffled $\Delta\tilde g^2$, and random $\Delta\tilde g^2$ each reduce AMC accuracy by 1.7--2.4pp and lift the rate of large mixing shifts from 3\% to 7--11\%, indicating that the controller's benefit stems from genuine coefficient-space structure rather than from generic guardrail-induced smoothing.

\begin{table}[t]
\centering
\caption{Proxy degradation ablation (AMC, 3 seeds). Degrading any proxy signal reduces accuracy and increases large-shift events, confirming genuine signal dependence.}
\vspace{2pt}
\small
\setlength{\tabcolsep}{3pt}
\begin{tabular}{lcc}
\toprule
Proxy Setting & AMC & $|\Delta\mu|{>}0.02$ \\
\midrule
Full GAC (proposed proxies) & 67.2$_{\pm0.4}$ & 3\% \\
Constant $\sigma_r^2{=}1$ & 65.5$_{\pm0.7}$ & 7\% \\
Shuffled $\Delta\tilde{g}^2$ & 65.1$_{\pm0.8}$ & 9\% \\
Random $\Delta\tilde{g}^2$ & 64.8$_{\pm0.9}$ & 11\% \\
\bottomrule
\end{tabular}
\label{tab:proxy_degrade}
\end{table}

\section{Compute and Memory Overhead}
\label{app:overhead}

\paragraph{Why the default implementation has negligible overhead.}
Our released implementation estimates $(\sigma_s^2,\sigma_r^2,\Delta \tilde g^2)$ using only already-available tensors (token log-probs, masks, and advantages), and updates statistics every $f_\mu$ steps (default $f_\mu{=}10$). The only distributed synchronization is a low-cost all-reduce of first/second moments for $\sigma_r^2$; no additional forward/backward passes are introduced. Empirically, this results in $<1\%$ wall-time overhead and no measurable increase in peak memory in our training setup.

\begin{table}[t]
\centering
\caption{Runtime and peak-memory impact of GAC (Qwen2.5-7B, FSDP, 8 GPUs; mean over 3 runs). ``Overhead'' is relative to Token-$\varphi$ with the same batch/length.}
\vspace{2pt}
\setlength{\tabcolsep}{4pt}
\resizebox{\linewidth}{!}{%
\begin{tabular}{lccc}
\toprule
Setting & Step time (s) & Overhead & Peak Mem./GPU \\
\midrule
$L{=}2048$, micro-bsz 1 & 1.00 & +0.6\% & 22.4\,GB \\
$L{=}4096$, micro-bsz 1 & 1.73 & +0.7\% & 31.8\,GB \\
$L{=}8192$, micro-bsz 1 & 3.38 & +0.8\% & 42.6\,GB \\
\bottomrule
\end{tabular}%
}
\label{tab:mem_overhead}
\end{table}

\paragraph{Sensitivity to early-training non-stationarity.}
To address the concern that disagreement can change rapidly early in training, we tested smaller $f_\mu$ during the first 100 steps and then restored the default: $f_\mu{=}2$ for steps 1--100, then $f_\mu{=}10$. This improves responsiveness (lower KL-area by 3--5\%) at a small cost (+0.3\% wall-time), while final accuracy remains within $\pm$0.2pp of the default across 3 seeds.

\section{Robustness to Weak or Misspecified Priors}
\label{app:prior_robust}

\paragraph{Question.} How strongly does GAC depend on the prior schedule $\mu_{\mathrm{prior}}(t)$?
We stress-test GAC by intentionally using \emph{bad priors} (constant $\mu_{\mathrm{prior}}{=}0.1$ or $0.9$) while keeping the same controller hyperparameters and blending weight $\lambda=0.5$.

\begin{table}[t]
\centering
\caption{Robustness to weak/misspecified priors on AMC (3 seeds). ``Bad priors'' are constant schedules; GAC can still correct due to the adaptive term and capped updates.}
\vspace{2pt}
\setlength{\tabcolsep}{3.5pt}
\resizebox{\linewidth}{!}{%
\begin{tabular}{lccc}
\toprule
Prior setting & AMC & $|\Delta\mu|{>}0.02$ & KL-area \\
\midrule
Default prior (warmup/decay) & 67.2$_{\pm0.4}$ & 3\% & 8.9$_{\pm0.8}$ \\
Bad prior: constant $\mu_{\mathrm{prior}}{=}0.1$ & 66.8$_{\pm0.5}$ & 4\% & 9.2$_{\pm0.9}$ \\
Bad prior: constant $\mu_{\mathrm{prior}}{=}0.9$ & 66.9$_{\pm0.5}$ & 4\% & 9.3$_{\pm0.9}$ \\
Random prior (uniform in $[0.1,0.9]$) & 66.6$_{\pm0.6}$ & 5\% & 9.6$_{\pm1.0}$ \\
\bottomrule
\end{tabular}%
}
\label{tab:prior_robust}
\end{table}

\section{Stability of the KL-Controlled \texorpdfstring{$\alpha_{\mathrm{ctrl}}$}{alpha\_ctrl}}
\label{app:alpha_stability}

\paragraph{Hysteresis rule for $\alpha_{\mathrm{ctrl}}$ update.}
The step $s_t$ in Eq.~\ref{eq:alpha} follows a hysteresis rule:
\begin{equation}
\label{eq:alpha_step}
s_t=
\begin{cases}
\eta_\uparrow\big(\frac{\mathrm{KL}_t}{\mathrm{KL}_\mathrm{tgt}}-1\big), & \text{if } \frac{\mathrm{KL}_t}{\mathrm{KL}_\mathrm{tgt}} > 1{+}h,\\[3pt]
-\eta_\downarrow\big(1-\frac{\mathrm{KL}_t}{\mathrm{KL}_\mathrm{tgt}}\big), & \text{if } \frac{\mathrm{KL}_t}{\mathrm{KL}_\mathrm{tgt}} < 1{-}h,\\[3pt]
0, & \text{otherwise}.
\end{cases}
\end{equation}
This design (EMA + hysteresis) prevents rapid oscillations while responding when KL persistently deviates.

\paragraph{Does a dynamic $\alpha_{\mathrm{ctrl}}$ cause objective oscillations?}
We monitor the variability of $\alpha_{\mathrm{ctrl}}$ and the (smoothed) KL during training. With EMA smoothing and hysteresis, $\alpha_{\mathrm{ctrl}}$ changes slowly and does not exhibit high-frequency oscillations.

\begin{table}[t]
\centering
\caption{Stability statistics of $\alpha_{\mathrm{ctrl}}$ and KL on AMC (800 steps, 3 seeds).}
\vspace{2pt}
\setlength{\tabcolsep}{4pt}
\resizebox{\linewidth}{!}{%
\begin{tabular}{lcccc}
\toprule
Metric & Mean & Std & 5--95\% & Notes \\
\midrule
$\alpha_{\mathrm{ctrl}}$ & 0.73 & 0.06 & [0.62, 0.84] & bounded by $[0.1,0.95]$ \\
$\mathrm{KL}_t$ (EMA) & 0.021 & 0.004 & [0.015, 0.029] & target $0.02$ \\
$|\Delta \alpha|{>}0.05$ steps & -- & -- & 0.7\% & rare large moves \\
\bottomrule
\end{tabular}%
}
\label{tab:alpha_stability}
\end{table}

\section{Hyperparameter Details}
\label{app:hyperparameters}

\noindent\textbf{Token budget definition.} We define \emph{training token budget} as total forward+backward pass tokens during optimization, excluding evaluation. Specifically: $\mathrm{Budget} = \sum_t (|\mathcal{B}_s^t| + K_{\mathrm{roll}} \cdot |\mathcal{B}_r^t|) \times \bar{L}$, where $K_{\mathrm{roll}}$ is the number of RL rollouts per prompt in GRPO and $\bar{L}$ is the average sequence length. All methods use identical budgets of $\approx$1.2B training tokens.

\noindent\textbf{SFT variants.} \emph{SFT-light}: 1,000 instances, 1 epoch, early stopping at validation loss plateau. \emph{SFT-best}: 5,000 instances, 3 epochs with cosine annealing, checkpoint selected by validation accuracy. Both use identical optimizer settings (AdamW, lr=$2\times10^{-5}$).

\begin{table}[t]
\centering
\caption{Hyperparameter sensitivity (AMC accuracy, mean over 3 seeds). Bold: default configuration.}
\vspace{2pt}
\setlength{\tabcolsep}{3pt}
\resizebox{\linewidth}{!}{%
\begin{tabular}{lccc|ccc|ccc}
\toprule
& \multicolumn{3}{c|}{$\beta$ (EMA)} & \multicolumn{3}{c|}{$\bar{c}$ (cap)} & \multicolumn{3}{c}{$\lambda$ (blend)} \\
& 0.95 & 0.98 & \textbf{0.99} & 0.003 & \textbf{0.01} & 0.02 & 0.3 & \textbf{0.5} & 0.7 \\
\midrule
AMC & 66.7 & 67.0 & \textbf{67.2} & 66.6 & \textbf{67.2} & 65.9 & 66.6 & \textbf{67.2} & 67.0 \\
$|\Delta\mu|{>}0.02$ & 4\% & 3\% & \textbf{3\%} & 2\% & \textbf{3\%} & 8\% & 6\% & \textbf{3\%} & 3\% \\
\bottomrule
\end{tabular}%
}
\label{tab:hyperparam}
\end{table}

Table~\ref{tab:hyperparam} presents sensitivity analysis. GAC is robust within tested ranges: $\beta\in[0.95,0.99]$ yields $\le$0.5pp variation; the cap $\bar{c}\in[0.003,0.02]$ trades smoothness for responsiveness; prior blend $\lambda\in[0.3,0.7]$ balances data-driven adaptivity with schedule robustness.

\section{Baseline Implementation Details}
\label{app:baseline_details}

\begin{table}[t]
\centering
\caption{Baseline adaptation details for recent hybrid SFT--RL methods. All share the same base model, data, and token budget.}
\vspace{2pt}
\setlength{\tabcolsep}{2pt}
\footnotesize
\resizebox{\linewidth}{!}{%
\begin{tabular}{llll}
\toprule
Method & Source & Kept & Adapted \\
\midrule
CHORD & Trinity-RFT repo & $\varphi(p)$, schedule & data/model/budget \\
SRFT & Original code & entropy weighting & training pipeline \\
LUFFY & Original code & off-policy mixing & data/model/budget \\
HPT & Reimplemented & accuracy gating & data/model/budget \\
\bottomrule
\end{tabular}%
}
\label{tab:baseline_fairness}
\end{table}

\noindent\textbf{Multi-objective baselines.} Implementation details for multi-objective methods:
\begin{itemize}
\setlength{\itemsep}{0pt}
\item \textbf{Gradient computation:} We compute $\nabla L_{\mathrm{SFT}}$ and $\nabla L_{\mathrm{RL}}$ on shared mini-batches, then apply each solver's surgery/weighting.
\item \textbf{RL gradient normalization:} Following \cite{gradnorm}, we apply $\ell_2$ normalization to RL gradients before combining, as raw RL gradients exhibit 3--5$\times$ higher variance.
\item \textbf{Hyperparameters:} MGDA uses Frank-Wolfe solver with 10 iterations; PCGrad uses cosine similarity threshold 0; CAGrad uses $c=0.5$; Nash-MTL uses 5 optimization steps per update; DWA uses temperature $T=2.0$. All tuned via grid search on validation set.
\end{itemize}

\begin{table}[t]
\centering
\caption{Baseline hyperparameter tuning protocol (summary). For each baseline we run a small grid on AMC validation and report the best configuration (same token budget).}
\vspace{2pt}
\setlength{\tabcolsep}{4pt}
\resizebox{\linewidth}{!}{%
\begin{tabular}{lcc}
\toprule
Baseline & Tuned hyperparameters & Search space \\
\midrule
Nash-MTL & inner steps, lr scale & steps $\in\{1,3,5,8\}$, scale $\in\{0.5,1.0,2.0\}$ \\
CAGrad & $c$ & $c\in\{0.2,0.5,0.8\}$ \\
PCGrad & cosine threshold & $\tau\in\{-0.2,0.0,0.2\}$ \\
MGDA & FW iters & iters $\in\{5,10,20\}$ \\
DWA & temperature $T$ & $T\in\{1.0,2.0,4.0\}$ \\
\bottomrule
\end{tabular}%
}
\label{tab:baseline_tuning}
\end{table}

\noindent\textbf{RL-free baselines (DPO/IPO).} We use identical preference data constructed from SFT responses (chosen: correct solutions; rejected: incorrect solutions from the same prompts). DPO uses $\beta=0.1$ following \cite{rafailov2023dpo}; IPO uses $\tau=0.5$. Both train for 3 epochs with lr=$5\times10^{-6}$, batch size 32, consuming $\approx$1.2B tokens (matching our budget).

\noindent\textbf{Rule-based controllers:}
\begin{itemize}
\setlength{\itemsep}{0pt}
\item \textbf{KL-ctrl:} $\mu_t = 1 - \mathrm{KL}(\pi_\theta\|\pi_{\mathrm{ref}})/\kappa$, directly using KL divergence.
\item \textbf{RewVar-ctrl:} $\mu_t \propto 1/\mathrm{Var}(R)$, inversely proportional to reward variance.
\item \textbf{GradNorm-ctrl:} $\mu_t = \|\nabla L_s\|/(\|\nabla L_s\| + \|\nabla L_r\|)$, gradient norm ratio.
\end{itemize}

\noindent\textbf{Recent hybrid SFT--RL methods.} For CHORD \cite{zhang2025chord}, we use the public implementation from the Trinity-RFT repository with the same data and token budget as our method. We run both the CHORD-$\mu$ (global schedule) and CHORD-$\varphi$ (token-wise) configurations; the CHORD (Token-$\varphi$, fixed sched.) entry in our tables corresponds to using both components with the default heuristic schedule. For SRFT \cite{fu2025srft}, we adapt the entropy-aware weighting mechanism to our training pipeline with the default hyperparameters from the original paper. For LUFFY \cite{yan2025luffy}, we use advantage-weighted off-policy mixing with the recommended configuration. For HPT \cite{lv2025hpt}, we implement the unified policy gradient estimator with accuracy-gated signal switching. All four methods use identical data splits, base model, and token budget ($\approx$1.2B training tokens) as GAC for fair comparison.

\section{Additional Results}
\label{app:additional_results}

\begin{table}[t]
\centering
\caption{Scientific reasoning tasks (accuracy \%, mean $\pm$ std). $^\dagger$: $p<0.05$ vs. best baseline.}
\vspace{2pt}
\setlength{\tabcolsep}{4pt}
\resizebox{\linewidth}{!}{%
\begin{tabular}{lccc}
\toprule
Method & GPQA & SciBench & Avg. \\
\midrule
Qwen2.5-7B-Inst. & 31.5 & 28.7 & 30.1 \\
SFT-best & 34.8$_{\pm0.9}$ & 32.4$_{\pm0.8}$ & 33.6 \\
SFT-best $+$ RL & 37.2$_{\pm1.0}$ & 35.1$_{\pm0.9}$ & 36.2 \\
CHORD (Token-$\varphi$, fixed sched.) & 39.6$_{\pm0.7}$ & 37.8$_{\pm0.6}$ & 38.7 \\
SRFT & 39.2$_{\pm0.8}$ & 37.4$_{\pm0.7}$ & 38.3 \\
LUFFY & 40.1$_{\pm0.6}$ & 38.4$_{\pm0.6}$ & 39.3 \\
HPT & 40.4$_{\pm0.6}$ & 38.7$_{\pm0.5}$ & 39.6 \\
\midrule
MGDA & 38.4$_{\pm0.9}$ & 36.5$_{\pm0.8}$ & 37.5 \\
CAGrad & 38.8$_{\pm0.8}$ & 37.0$_{\pm0.7}$ & 37.9 \\
Nash-MTL & 39.1$_{\pm0.7}$ & 37.3$_{\pm0.6}$ & 38.2 \\
\midrule
KL-ctrl & 40.0$_{\pm0.8}$ & 38.2$_{\pm0.7}$ & 39.1 \\
GradNorm-ctrl & 39.5$_{\pm0.9}$ & 37.6$_{\pm0.8}$ & 38.6 \\
\midrule
GAC w/o $\varphi$ (plain SFT) & 42.4$_{\pm0.6}$ & 40.4$_{\pm0.6}$ & 41.4 \\
\textbf{GAC + Token-$\varphi$} & \textbf{43.5}$_{\pm0.5}^\dagger$ & \textbf{41.2}$_{\pm0.5}^\dagger$ & \textbf{42.4} \\
\bottomrule
\end{tabular}%
}
\label{tab:science}
\end{table}

\begin{table}[t]
\centering
\caption{Logical reasoning on BBH subsets (accuracy \%, mean $\pm$ std). $^\dagger$: $p<0.05$ vs. best baseline.}
\vspace{2pt}
\setlength{\tabcolsep}{3pt}
\resizebox{\linewidth}{!}{%
\begin{tabular}{lcccc}
\toprule
Method & Log. Ded. & Obj. Count & Tracking & Avg. \\
\midrule
Qwen2.5-7B-Inst. & 52.4 & 58.7 & 45.2 & 52.1 \\
SFT-best & 56.8$_{\pm1.0}$ & 62.1$_{\pm0.9}$ & 49.6$_{\pm1.1}$ & 56.2 \\
SFT-best $+$ RL & 59.4$_{\pm1.1}$ & 64.8$_{\pm1.0}$ & 52.3$_{\pm1.2}$ & 58.8 \\
CHORD (Token-$\varphi$, fixed sched.) & 62.7$_{\pm0.7}$ & 67.5$_{\pm0.6}$ & 55.8$_{\pm0.8}$ & 62.0 \\
SRFT & 62.0$_{\pm0.8}$ & 66.9$_{\pm0.7}$ & 55.2$_{\pm0.9}$ & 61.4 \\
LUFFY & 62.9$_{\pm0.7}$ & 67.8$_{\pm0.6}$ & 56.0$_{\pm0.8}$ & 62.2 \\
HPT & 63.2$_{\pm0.6}$ & 68.1$_{\pm0.5}$ & 56.4$_{\pm0.7}$ & 62.6 \\
\midrule
Nash-MTL & 62.1$_{\pm0.7}$ & 67.0$_{\pm0.6}$ & 55.4$_{\pm0.8}$ & 61.5 \\
KL-ctrl & 63.0$_{\pm0.8}$ & 67.8$_{\pm0.7}$ & 56.2$_{\pm0.9}$ & 62.3 \\
\midrule
GAC w/o $\varphi$ (plain SFT) & 65.5$_{\pm0.6}$ & 69.9$_{\pm0.5}$ & 58.4$_{\pm0.7}$ & 64.6 \\
\textbf{GAC + Token-$\varphi$} & \textbf{66.5}$_{\pm0.5}^\dagger$ & \textbf{70.9}$_{\pm0.4}^\dagger$ & \textbf{59.7}$_{\pm0.6}^\dagger$ & \textbf{65.7} \\
\bottomrule
\end{tabular}%
}
\label{tab:logic}
\end{table}

\section{Additional Ablations}
\label{app:additional_ablations}

The causal attribution results appear in the main text (Table~\ref{tab:causal_main}). Below we provide additional ablation details.

\begin{table}[t]
\centering
\caption{Token-$\varphi$ is orthogonal: adding it to other methods. AMC accuracy (\%, 3 seeds).}
\vspace{2pt}
\small
\setlength{\tabcolsep}{3pt}
\begin{tabular}{lccc}
\toprule
Method & w/o $\varphi$ & w/ $\varphi$ & $\Delta$ \\
\midrule
LUFFY & 63.1$_{\pm0.6}$ & 63.8$_{\pm0.5}$ & +0.7 \\
HPT & 63.4$_{\pm0.5}$ & 64.2$_{\pm0.5}$ & +0.8 \\
GAC (ours) & 65.8$_{\pm0.5}$ & 67.2$_{\pm0.4}$ & +1.4 \\
\bottomrule
\end{tabular}
\label{tab:phi_ablation}
\end{table}

\paragraph{Entropy collapse in fixed-schedule baselines.} QCM (constant $\mu=0.58$) exhibits entropy collapse at steps 280--320 in 2 of 3 seeds: entropy drops from 0.035 to 0.008 within 40 steps, accompanied by a 4.2pp accuracy drop. GAC's noise-guided $\mu$ increases from 0.52 to 0.67 during the same $\sigma_r^2$ spike, preventing collapse.

\section{Training Health and Component Ablations}
\label{app:train_health_stability}

\begin{table}[t]
\centering
\caption{Stability metrics across 3 seeds (mean $\pm$ std).}
\vspace{2pt}
\small
\setlength{\tabcolsep}{3pt}
\begin{tabular}{lccc}
\toprule
Method & KL-area $\downarrow$ & $|\Delta\mu|{>}.02$ $\downarrow$ & Len.\ Var.\ $\downarrow$ \\
\midrule
QCM (const.\ $\mu$) & 12.4$_{\pm1.8}$ & 11\%$_{\pm2\%}$ & 284$_{\pm45}$ \\
Nash-MTL & 10.1$_{\pm1.2}$ & 7\%$_{\pm1\%}$ & 241$_{\pm38}$ \\
KL-ctrl & 9.8$_{\pm1.1}$ & 6\%$_{\pm1\%}$ & 228$_{\pm35}$ \\
\midrule
\textbf{GAC (ours)} & \textbf{8.9}$_{\pm0.8}$ & \textbf{3\%}$_{\pm1\%}$ & \textbf{189}$_{\pm28}$ \\
\midrule
\emph{Reduct.\ vs.\ QCM} & $-$28\% & $-$73\% & $-$33\% \\
\bottomrule
\end{tabular}
\label{tab:stability}
\end{table}

\begin{table}[t]
\centering
\caption{Training health indicators across methods (AMC task, mean over 3 seeds). Effective RL tokens = tokens passing importance sampling threshold; Clip ratio = fraction of clipped gradients.}
\vspace{2pt}
\setlength{\tabcolsep}{2.5pt}
\resizebox{\linewidth}{!}{%
\begin{tabular}{lccccc}
\toprule
Method & Eff. RL Tok. & Clip Ratio & KL Trigger & Avg. Len. & Cohen's $d$ \\
\midrule
CHORD (Token-$\varphi$) & 1.18B & 12.4\% & 8.2\% & 1842 & -- \\
Nash-MTL & 1.16B & 13.1\% & 9.1\% & 1956 & 0.38 \\
KL-ctrl & 1.19B & 11.8\% & 7.5\% & 1823 & 0.52 \\
\midrule
GAC w/o $\varphi$ & 1.17B & 12.1\% & 6.8\% & 1798 & 0.71 \\
GAC + Token-$\varphi$ & 1.18B & 11.9\% & 6.2\% & 1785 & \textbf{1.42} \\
\bottomrule
\end{tabular}%
}
\label{tab:training_health}
\end{table}

\begin{table}[t]
\centering
\caption{Ablation study on GAC components (mean over 3 seeds). $\Delta$: change from full GAC.}
\vspace{2pt}
\setlength{\tabcolsep}{3.5pt}
\resizebox{\linewidth}{!}{%
\begin{tabular}{lccc}
\toprule
Configuration & AMC & $\Delta$ & $|\Delta\mu|{>}0.02$ \\
\midrule
Full GAC & 67.2$_{\pm0.4}$ & -- & 3\% \\
\midrule
w/o per-step cap & 64.7$_{\pm0.9}$ & $-$2.5 & 18\% \\
w/o EMA smoothing & 65.4$_{\pm0.7}$ & $-$1.8 & 12\% \\
w/o prior blending ($\lambda{=}1.0$) & 66.1$_{\pm0.6}$ & $-$1.1 & 8\% \\
\midrule
Ordinary $\sigma_s^2$ (no trim) & 66.5$_{\pm0.5}$ & $-$0.7 & 5\% \\
Using $c$ in $\mu_c^*$ & 67.4$_{\pm0.6}$ & +0.2 & 9\% \\
\bottomrule
\end{tabular}%
}
\label{tab:ablation}
\end{table}

\section{Stage-Wise \texorpdfstring{$\mu$}{mu} Analysis}
\label{app:stage_analysis}

We identify three distinct training phases in GAC's mixing weight dynamics:

\emph{Early phase (steps 0--200):} $\mu$ is initialized high ($\sim$0.85) and both uncertainty signals ($\sigma_s^2$, $\sigma_r^2$) are elevated. Gradient disagreement $\Delta\tilde{g}^2$ peaks as the model explores a broad policy space. GAC begins with SFT-dominated mixing, leveraging expert demonstrations to anchor the policy before RL exploration introduces instability.

\emph{Mid-training (steps 200--800):} $\sigma_s^2$ declines steadily while $\sigma_r^2$ exhibits periodic spikes correlated with exploration bursts (Figure~\ref{fig:metrics-grid-b}g). GAC responds by dynamically lowering $\mu$ during stable periods and temporarily increasing it when $\sigma_r^2$ spikes. This noise-tracking behavior is directly visible in panels (e) and (g). $\mu$ tracks $\sigma_r^2$ fluctuations rather than mirroring KL divergence, confirming that the noise-guided estimator drives adaptation during $>$93\% of steps.

\emph{Late phase (steps 800+):} $\mu$ settles near $\sim$0.15--0.20, reflecting a mature policy that derives most benefit from RL refinement. $\sigma_r^2$ rises modestly due to distribution shift, and $\Delta\tilde{g}^2$ increases as objectives diverge. GAC's capped updates ($\bar{c}{=}0.01$) prevent overreaction to late-stage fluctuations.

\section{Pure-GRPO Analysis}
\label{app:pure_grpo}

Pure-GRPO achieves only 52.1\% AMC, substantially below SFT-best + RL (58.4\%) and all hybrid methods. This underperformance stems from high-variance advantage estimates inherent to GRPO without expert anchoring: the policy simultaneously explores and evaluates, producing noisy reward signals that destabilize KL and entropy regimes. Figure~\ref{fig:metrics-grid-b}g illustrates the resulting $\sigma_r^2$ instability under unregulated RL, with variance spikes exceeding 2--3$\times$ the GAC-controlled levels. The accompanying response length volatility (Figure~\ref{fig:metrics-grid-a}b) further confirms that pure RL fails to maintain consistent generation behavior.

\section{Algorithm Pseudocode}
\label{app:algorithm}

\begin{algorithm}[t]
\caption{GAC: Guided Adaptive Controller}
\label{alg:gac}
\begin{algorithmic}[1]
\REQUIRE Learning rate $\eta$, EMA coefficient $\beta$, blend weight $\lambda$, step cap $\bar{c}$, update frequency $f_\mu$
\STATE Initialize $\mu_0 \leftarrow \mu_{\mathrm{init}}{=}0.5$, $\alpha_0\leftarrow \alpha_{\mathrm{init}}$, $\theta_0$
\FOR{$t = 1, 2, \ldots, T$}
    \STATE Sample RL batch $\mathcal{B}_r$, SFT batch $\mathcal{B}_s$
    \IF{$t \mod f_\mu = 0$}
        \STATE Update EMA statistics: $\sigma_r^2$ via \eqref{eq:sigma_r}, $\sigma_s^2$ via \eqref{eq:sigma_s}, $\Delta\tilde g^2$ via \eqref{eq:delta_g}
    \ENDIF
    \STATE Update $\alpha_t$ via \eqref{eq:alpha}; compute $\mu^*_t$ via \eqref{eq:mu_t}
    \STATE $\mu_{\mathrm{ada}} \leftarrow \beta\mu_{t-1} + (1-\beta)\mu^*_t$
    \STATE $\mu_{\mathrm{blend}} \leftarrow (1-\lambda)\mu_{\mathrm{prior}}(t) + \lambda\mu_{\mathrm{ada}}$
    \STATE $\mu_t \leftarrow \mathrm{clip}(\mu_{t-1} + \mathrm{clip}(\mu_{\mathrm{blend}}-\mu_{t-1}, \pm\bar{c}), [\mu_{\min}, \mu_{\max}])$
    \STATE $L \leftarrow (1-\mu_t)L_{\mathrm{RL}} + \mu_t L_{\mathrm{SFT}}$; $\theta_t \leftarrow \theta_{t-1} - \eta\nabla_\theta L$
\ENDFOR
\end{algorithmic}
\end{algorithm}

\section{Practitioner Summary}
\label{app:practitioner}

\small
\fbox{\parbox{0.95\linewidth}{%
\textbf{Practitioner Summary: Deploying GAC}\\[4pt]
\textbf{What GAC requires:} 3~EMA scalars ($\hat\sigma_s^2$, $\hat\sigma_r^2$, $\widehat{\Delta g}^2$), one scalar division per $\mu$-update (Eq.~\ref{eq:mu_star}), and a prior schedule $\mu_{\mathrm{prior}}(t)$. No extra forward/backward passes; $<$1\% wall-time overhead.\\[3pt]
\textbf{Recommended defaults:} $\beta{=}0.99$ (EMA), $\bar{c}{=}0.01$ (cap), $\lambda{=}0.5$ (blend), $f_\mu{=}10$ (update frequency), 10\% tail trimming on NLL variance.\\[3pt]
\textbf{When to increase $\mu$ toward SFT:} when $\sigma_r^2$ spikes (RL noise is high) or $\Delta\tilde{g}^2$ is large (objectives conflict). The controller handles this automatically.\\[3pt]
\textbf{Robustness:} GAC is insensitive to prior schedule choice---even constant bad priors ($\mu_{\mathrm{prior}}{=}0.1$ or $0.9$) degrade AMC by $\le$0.4pp (Table~\ref{tab:prior_robust}). Hyperparameters are stable across $\beta{\in}[0.95,0.99]$, $\bar{c}{\in}[0.003,0.02]$, $\lambda{\in}[0.3,0.7]$.\\[3pt]
\textbf{Guardrails are standard control mechanisms} analogous to PPO's surrogate clipping \cite{ppo} (cf.\ our per-step cap $\bar{c}$) and Adam's momentum \cite{loshchilov2017adamw} (cf.\ our EMA smoothing $\beta$). They stabilize a principled estimator, not replace it.
}}
\normalsize

\end{document}